\documentclass[journal, twoside]{IEEEtran}

\usepackage{microtype}
\usepackage{wrapfig}
\usepackage{stfloats}
\usepackage{moresize}
\usepackage{graphicx}
\usepackage{booktabs} 
\usepackage[utf8]{inputenc}
\usepackage{hyperref}
\usepackage{makecell}
\usepackage{url}
\usepackage{caption,subcaption}
\usepackage{array}
\usepackage[rflt]{floatflt}
\usepackage{epsfig,times,graphicx,color}
\usepackage{amsmath}
\usepackage{amssymb, amsopn,algorithm,algorithmic,float,bbm,bm,enumerate,color,multirow}
\usepackage{rotating}
\usepackage{array}
\usepackage{multirow}
\usepackage{hhline}
\usepackage{xspace}
\usepackage{microtype}
\usepackage{graphicx}
\usepackage{booktabs} 
\usepackage[utf8]{inputenc} 
\usepackage[T1]{fontenc}    
\usepackage{hyperref}       
\usepackage{url}            
\usepackage{booktabs}       
\usepackage{amsfonts}       
\usepackage{nicefrac}       
\usepackage{microtype}      
\usepackage{enumitem}
\usepackage{calc}

\newcommand{\dc}{{\sc Dasher}}
\def \ds{{\sc Dashin}}

\usepackage{lipsum}
\usepackage{amsfonts}
\usepackage{graphicx}
\usepackage{epstopdf}
\usepackage{algorithmic}

\usepackage{enumitem}
\usepackage{amsopn}

\newcommand{\Ic}{{\cal{I}}}
\newcommand{\ratio}{\bar{\rho}}
\newcommand{\lamin}{\lambda_{\min}}

\newcommand{\rinc}{\psi_{\Ic}}

\newcommand{\norm}[2]{\|#1\|_{#2}}

\newcommand{\normth}[2]{\vertiii{#1}_{#2}}

\newcommand{\beq}{\begin{equation}}
\newcommand{\eeq}{\end{equation}}

\newcommand\qedsymbol{$\blacksquare$}

\newcommand\B{\mathbf{B}}

\newcommand\I{\mathbf{I}}

\newcommand\V{\mathbf{V}}

\newcommand\X{\mathbf{X}}

\newcommand\0{\mathbf{0}}

\newcommand{\g}{\mathbf{g}}

\renewcommand{\b}{\mathbf{b}}

\newcommand{\h}{\mathbf{h}}
\renewcommand{\u}{\mathbf{u}}
\renewcommand{\v}{\mathbf{v}}

\newcommand{\w}{\mathbf{w}}
\newcommand{\x}{\mathbf{x}}
\newcommand{\y}{\mathbf{y}}

\newcommand{\p}{\mathbf{p}}
\newcommand{\q}{\mathbf{q}}

\newcommand{\cA}{{\cal A}}
\newcommand{\cB}{{\cal B}}
\newcommand{\cC}{{\cal C}}
\newcommand{\cE}{{\cal E}}

\newcommand{\cH}{{\cal H}}
\newcommand{\cI}{{\cal I}}
\newcommand{\cJ}{{\cal J}}
\newcommand{\cK}{{\cal K}}

\newcommand{\cS}{{\cal S}}

\newcommand{\cV}{{\cal V}}

\newcommand{\cX}{{\cal X}}

\newcommand{\hbbe}{\hat{\bbeta}}

\newcommand{\pr}{\mathbb{P}}
\newcommand{\ex}{\mathbb{E}}

\newcommand{\vertiii}[1]{{\left\vert\kern-0.25ex\left\vert\kern-0.25ex\left\vert #1
    \right\vert\kern-0.25ex\right\vert\kern-0.25ex\right\vert}}

\newcommand{\oomega}{\boldsymbol{\omega}}

\newcommand{\bbeta}{\boldsymbol{\beta}}

\newcommand{\ddelta}{\boldsymbol{\delta}}
\newcommand{\mmu}{\boldsymbol{\mu}}

\DeclareMathOperator*{\argmin}{\arg\!\min}
\newcommand {\commentout}[1] {}

\def\ints{{{\rm Z} \kern -.35em {\rm Z} }}  
\def\smallints{{{\rm Z} \kern -.3em {\rm Z} }}  
\def\pints{{{\rm I} \kern -.15em {\rm N} }}      
\newcommand{\reals}{\mathbb R}
\newcommand{\sphere}{\mathbb S^{p-1}}
\newcommand{\ball}{\mathbb B^p}

\def\cplx{{{\rm I} \kern -.45em {\rm C} }}       
\def\l2{\rm {\mathcal L}^{2}(\reals)}            

\newtheorem{theorem}{Theorem}
\newtheorem{remark}{Remark}
\newtheorem{definition}{Definition}
\newtheorem{lemma}[theorem]{Lemma}
\newtheorem{prop}[theorem]{Proposition}
\newtheorem{corollary}[theorem]{Corollary}
\newtheorem{example}{Example}

\newcommand{\nr}{\nonumber}
\newcommand{\be}{\begin{eqnarray}}
\newcommand{\ee}{\end{eqnarray}}

\def\bea#1\eea{\begin{align}#1\end{align}}

\newcommand{\beaa}{\begin{eqnarray*}}
\newcommand{\eeaa}{\end{eqnarray*}}
\newcommand{\bnad}{\begin{nad}}
\newcommand{\enad}{\end{nad}}

\newcommand{\indic}{\mathbbm{1}}

\newcommand{\bcH}{\bar{\cH}} 

\begin{document}

\title{High Dimensional Data Enrichment: \\ Interpretable, Fast, and Data-Efficient}

\author{Amir Asiaee,
        Samet Oymak,
        Kevin R. Coombes, 
        and Arindam Banerjee
\thanks{The research was supported in part by NSF grants OAC-1934634, IIS-1908104, IIS-1563950, IIS-1447566, IIS-1447574, IIS-1422557, and CCF-1451986. The work of S. Oymak is partially supported by the NSF award CNS-1932254. The work of A. Asiaee was partially supported by MBI (Mathematical Biosciences Institute) which receives its funding through NSF grant DMS-1440386.}
\thanks{A. Asiaee is with Mathematical Biosciences Institute The Ohio State University, Columbus, OH, 43210 USA (e-mail: asiaeetaheri.1@osu.edu)}	
\thanks{S. Oymak is with Department of Electrical and Computer Engineering, UC Riverside, Riverside, CA, 92507 USA (e-mail: oymak@ece.ucr.edu)}	
\thanks{K. R. Coombes is with Departments of Biomedical Informatics, The Ohio State University, Columbus, OH, 43210 USA (e-mail: coombes.3@osu.edu)}	
\thanks{A. Banerjee is with Department of Computer Science and Engineering, University of Minnesota, Minneapolis, MN, 55455 USA (e-mail: banerjee@cs.umn.edu)}		
 }

\markboth{ }
{Asiaee \MakeLowercase{\textit{et al.}}: High Dimensional Data Enrichment: Interpretable, Fast, and Data-Efficient}

\maketitle

\begin{abstract}
	We consider the problem of multi-task learning in the high dimensional setting. In particular, we introduce an estimator and investigate its statistical and computational properties for the problem of multiple connected linear regressions known as Data Sharing. The between-tasks connections are captured by a cross-tasks \emph{common parameter}, which gets refined by per-task \emph{individual parameters}. Any convex function, e.g., norm, can characterize the structure of both common and individual parameters.	We delineate the sample complexity of our estimator and provide high probability non-asymptotic bound for estimation error of all parameters under a geometric condition. We show that the recovery of the common parameter benefits from \emph{all} of the pooled samples. We propose an iterative estimation algorithm with a geometric convergence rate and supplement our theoretical analysis with experiments on synthetic data. Overall, we present a first through statistical and computational analysis of inference in the data sharing model. 	
\end{abstract}


\IEEEpeerreviewmaketitle

\section{Introduction}
Over the past two decades, major advances have been made in estimating structured parameters, e.g., sparse, low-rank, etc., in high-dimensional small sample problems \cite{candes2010power, donoho2006compressed, friedman2008sparse}. Such estimators consider a suitable (semi) parametric model of the response: $y = \phi(\x,\bbeta^*) + w$ based on $n$ samples $\{(\x_i,y_i)\}_{i=1}^n$ and the parameter of interest, $\bbeta^* \in \reals^p$. The unique aspect of such high-dimensional regime of $n \ll p$ is that the structure of $\bbeta^*$ makes
the estimation possible for large enough samples $n = m$ known as the sample complexity \cite{candes2009exact, candes2006robust, tibshirani1996regression}. While the earlier developments in such high-dimensional estimation problems had focused on parametric linear models, the results have been widely extended to non-linear models, e.g., generalized linear models \cite{bach2012optimization, negahban2009unified}, broad families of semi-parametric and single-index models \cite{boufounos20081, plan2017high}, non-convex models \cite{blumensath2009iterative,jain2013low}, etc.

In several real world problems, the assumption that one global model parameter $\bbeta_0^*$ is suitable for the entire population is unrealistic.
We consider the more general setting where the population consists of sub-populations (groups) which are similar is many aspects but have unique differences. For example, in the context of anti-cancer drug sensitivity prediction where the goal is to predict responses of different tumor cells to a drug, using the same prediction model across cancer types (groups) ignores the unique properties of each cancer and leads to an uninterpretable global model. Alternatively, in such a setting, one can assume a separate model for each group $g$ as $y = \phi(\x, \bbeta_g^*) + w$ based on a group specific parameter $\bbeta^*_g$. Such a modeling choice fails to leverage the similarities across the sub-populations, and can only be estimated when sufficient number of samples are available for each group which is not the case in several problems, e.g., anti-cancer drug sensitivity prediction \cite{barretina2012cancer, iorio2016landscape1}.

The middle ground model for such a scenario is the \emph{superposition} of common and individual parameters $\bbeta_0^* + \bbeta_g^*$ which has been of recent interest \cite{guba16, mctr13, Yang2013-pf}. Such a collection of \emph{coupled} superposition models is known by multiple names.
It is a form of multi-task learning \cite{jrsr10, Zhang2017-rm} when we consider regression in each group as a task. It is also called data sharing \cite{grti16} since information contained in different groups is shared through the common parameter $\bbeta_0^*$. Finally, it has been called data enrichment \cite{asiaee2018high, asiaeedata, Chen2015-fj} because we enrich our data set with pooling multiple samples from different but related sources. 

Following the successful application of such a modeling scheme in recent years \cite{domu16, grti16,  olvi14, olvi15}, we consider the below \emph{data sharing} (DS) model: 
\beq
\label{eq:dsl}
y_{gi} = \phi(\x_{gi}, (\bbeta_0^* + \bbeta^*_g)) + w_{gi}, \quad g \in \{1, \dots, G\}~,
\eeq
where $g$ and $i$ index the group and samples respectively. 
The DS model \eqref{eq:dsl} assumes that there is a \emph{common} parameter $\bbeta^*_0$ shared between all groups which models similarities between all samples. And there are \emph{individual} per-group parameters $\bbeta_g^*$s each characterize the deviation of group $g$.

Our goal is to design an estimation procedure which consistently recovers all parameters of DS \eqref{eq:dsl} fast and with small number of samples.
We specifically focus on the high-dimensional small sample regime where the number of samples $n_g$ for each group is much smaller than the ambient 
dimensionality, i.e., $\forall g: n_g \ll p$. Similar to all other high-dimensional models, we assume that the parameters $\bbeta_g$ are structured, i.e., for suitable convex functions $f_g$'s, $f_g(\bbeta_g)$ is small.
For example, when the structure is sparsity, $f_g$s are $l_1$-norms. Further, for the technical analysis and proofs,
we focus on the case of linear models, i.e., $\phi(\x,\bbeta) = \x^T \bbeta$. The results
seamlessly extend to more general non-linear models, e.g., generalized linear models, broad families of semi-parametric and single-index models, non-convex models, etc., using
existing results, i.e., how models like LASSO have been extended to these settings \cite{Kakade2010-st, negahban2012restricted, Plan2013-nx, Plan2016-de, Yang2016-zd}. 


\subsection{Related Work}
In the context of \emph{Multi-Task Learning} (MTL), similar models have been proposed which have the general form of $y_{gi} = \x_{gi}^T (\bbeta_{1g}^* + \bbeta^*_{2g}) + w_{gi}$ where $\B_1 = [\bbeta_{11}, \dots, \bbeta_{1G}]$ and $\B_2 = [\bbeta_{21}, \dots, \bbeta_{2G}]$ are two parameter matrices \cite{Zhang2017-rm}. To capture the relation between tasks, different types of constraints are assumed for parameter matrices. For example, \cite{Chen2012-fb} assumes $\B_1$ and $\B_2$ are sparse and low rank respectively. In this parameter matrix decomposition framework for MLT, the most related work to ours is the Dirty Statistical Model (DSM) proposed in \cite{jrsr10} where authors regularize the regression with $\norm{\B_1}{1, \infty}$ and $\norm{\B_2}{1,1}$ where norms are $i, j$-norms on \emph{rows}, $\b$, of matrices, i.e., $\norm{\B}{i,j} = \norm{(\norm{\b_1}{j}, \dots, \norm{\b_p}{j})}{i}$ and the norms are defined as $\norm{\b}{i} =$ {\small$\left(\sum_{g=1}^{G}|b_g|^i\right)^{1/i}$} and $\norm{\b}{\infty} = \max_{g\in G} |b_g|$. 

If in our DS model we pick all structure inducing functions $f_g$ to be $l_1$-norm, the resulting model is very similar to the DSM where $\norm{\B_1}{1, \infty}$ induces similarity between tasks and $\norm{\B_2}{1,1}$ models their discrepancies. On the other hand, the degree of freedom of DSM model is higher than DS because $\norm{\B_1}{1, \infty}$ regularizer enforces shared support of $\bbeta^*_{1g}$s, i.e., $\text{supp}(\bbeta_{1i}^*) = \text{supp}(\bbeta_{1j}^*)$ but allows $\bbeta_{1i}^* \neq \bbeta_{1j}^*$ while in DS we have a single common parameter $\bbeta_0^*$. So one would expect that DS estimators should have smaller sample complexity compared to their DSM counterparts and our analysis confirm that our estimator is more data efficient than DSM estimator of \cite{jrsr10}, Table \ref{compare}. Mainly, DSM requires every task $i$ to have large enough samples to learn its own common parameters $\bbeta_i$ but since DS shares the common parameter it only requires the {\em{total dataset over all tasks}} to be sufficiently large.

The linear DS model where $\bbeta_g$'s are sparse has recently gained attention because of its application in wide range of domains such as personalized medicine \cite{domu16}, sentiment analysis, banking strategy \cite{grti16}, single cell data analysis \cite{olvi15}, road safety \cite{olvi14}, and disease subtype analysis \cite{domu16}.
More generally, in any high-dimensional problem where the population consists of groups, data sharing has the potential to boost the prediction accuracy and results in a more interpretable set of parameters.

In spite of the recent surge in applying data sharing framework to different domains, limited advances have been made in
understanding the statistical and computational properties of suitable estimators for the DS model \eqref{eq:dsl}.
In fact, non-asymptotic statistical properties, including sample complexity and statistical rates of convergence, of regularized estimators for the data sharing model is still an open question \cite{grti16, olvi14}.
To the best of our knowledge, the only theoretical guarantee for data sharing is provided in \cite{olvi15} where authors prove sparsistency of their proposed method under the irrepresentability condition of the design matrix for recovering \emph{supports} of common and individual parameters.
Existing support recovery guarantees \cite{olvi15}, sample complexity and $l_2$ consistency results \cite{jrsr10} of related MTL models are restricted to sparsity and $l_1$-norm, while our estimator and \emph{norm consistency} analysis work for \emph{any} structure induced by arbitrary convex functions $f_g$. 
Moreover, no computational results, such as rates of convergence of the estimation procedures exist in the literature.

\subsection{Notation and Preliminaries}
We denote sets by curly $\cV$, matrices by bold capital $\V$, random variables by capital $V$, and vectors by small bold $\v$ letters.
We take $[G] = \{1, \dots, G\}$ and $[G_+] = [G] \cup \{0\}$. Throughout the manuscript $c_i$ and $C_i$ denote positive absolute constants.
Given $G$ groups and $n_g$ samples in each as $\{ \{\x_{gi}, y_{gi} \}_{i=1}^{n_g} \}_{g = 1}^G$, we can form the per group design matrix $\X_g \in \reals^{n_g \times p}$ and output vector $\y_g \in \reals ^{n_g}$.
The total number of samples is  $n = \sum_{g = 1}^{G} n_g$ and the data sharing model takes the following vector form:
\beq \label{eq:dirtymodel}
\y_g = \X_g (\bbeta _0^* + \bbeta _g^*) + \w_g,  \quad \forall g \in [G]
\eeq
where each row of $\X_g$ is $\x_{gi}^T$ and $\w_g^T = (w_{g1}, \dots, w_{gn_g})$ is the noise vector. It is useful for indexing to consider the common parameter as the zeroth group and define $n_0 \triangleq n$ and $\X_0 \triangleq [\X_1^T, \dots, \X_G^T]^T$.

\textbf{Sub-Gaussian random variable and vector:}
A random variable $V$ is sub-Gaussian if its moments satisfies $\forall p \geq 1: (\ex |V|^p )^{1/p} \leq K_2 \sqrt{p}$.
The minimum value of $K_2$ is called the sub-Gaussian  norm of $V$, denoted by $\normth{V}{\psi_2}$ \cite{vers12}.
A random vector $\v \in \reals^p$ is sub-Gaussian if the one-dimensional marginals $\langle \v, \u \rangle$ are sub-Gaussian random variables for all $\u \in \reals^p$. The sub-Gaussian norm of $\v$ is defined \cite{vers12} as $\normth{\v}{\psi_2} = \sup_{\u \in \sphere} \normth{\langle \v, \u \rangle}{\psi_2}$.
For any set $\cV \in \reals^p$ the Gaussian width of the set $\cV$ is defined as $\omega(\cV) = \ex_\g \left[ \sup_{\u \in \cV} \langle \g, \u \rangle \right]$ \cite{vershynin2018high}, where the expectation is over $\g \sim N(\0, \I_{p \times p})$, a vector of independent zero-mean unit-variance Gaussian. The marginal tail function is defined as $Q_{\xi}(\u) = \pr(|\langle \x, , \u \rangle| > \xi)$ for a fixed vector $\u$, random vector $\x$ and constant $\xi > 0$.

\subsection{Our Contributions}
We propose the following Data Shared (DS) estimator $\hbbe$ for recovering the structured parameters where the structure is induced by \emph{convex} functions $f_g(\cdot)$:
\begin{align}
	\label{eq:super}
	\hbbe = (\hbbe_0^T, \dots, \hbbe_G^T) &\in \argmin_{\bbeta _0, \dots, \bbeta _G} \frac{1}{n} \sum_{g=1}^{G} \norm{\y_g - \X_g (\bbeta _0 + \bbeta _g)}{2}^2, 
	\\ \nr 
	&~\text{s.t.} ~ \forall g \in [G_+]:f_g(\bbeta _g) \leq f_g(\bbeta _g^*).
\end{align}
We present several statistical and computational results for the DS estimator \eqref{eq:super}:
\begin{itemize}[leftmargin = .4cm]
	\item The DS estimator \eqref{eq:super} succeeds if a geometric condition that we call \emph{DAta SHaring Incoherence conditioN} (\ds) is satisfied, Fig. \ref{fig:DASHIN}. Compared to other known geometric conditions in the literature such as structural coherence \cite{guba16} and stable recovery conditions \cite{mctr13}, \ds\ is a considerably weaker condition, Fig. \ref{fig:sc}.
	\item Assuming \ds\ holds, we establish a high probability non-asymptotic bound on the weighted sum of parameter-wise estimation error, $\ddelta_g = \hbbe_g - \bbeta_g^*$ as:
	\beq
	\label{eq:errorsum}
	\sum_{g=0}^{G}  \sqrt{\frac{n_g}{n}} \|\ddelta_g\|_2 \leq  \gamma O\left(\frac{\max_{g \in [G]} \omega(\cC_g \cap \sphere)}{\sqrt{n}}\right),
	\eeq
	where $n_0 \triangleq n$ is the total number of samples, $\gamma \triangleq \max_{g \in [G] } \frac{n}{n_g}$ is the \emph{sample condition number}, and $\cC_g$ is the error cone corresponding to $\bbeta_g^*$ exactly defined in Section \ref{sec:esti}.
	To the best of our knowledge, this is the first statistical estimation guarantee for the data sharing.
	
	\item We also establish the sample complexity of the DS estimator for all parameters as $\forall g \in [G_+]: n_g = O(\omega(\cC_g \cap \sphere))^2$. We emphasize that our result proofs that the recovery of the common parameter $\bbeta_0$ by DS estimator \eqref{eq:super} benefits from \emph{all} of the $n$ pooled samples.
	\item We present an efficient projected block gradient descent algorithm \emph{\dc}, to solve DE's objective \eqref{eq:super} which converges geometrically to the statistical error bound of \eqref{eq:errorsum}. To the best of our knowledge, this is the first rigorous computational result for the high-dimensional data-shared regression.
\end{itemize}

The rest of this paper is organized as follows:
First, we characterize the error set of our estimator and provide a deterministic error bound in Section \ref{sec:esti}.
Then in Section \ref{sec:re}, we discuss the restricted eigenvalue condition and calculate the sample complexity required for the recovery of the true parameters by our estimator under \ds\ condition.
We close the statistical analysis in Section \ref{sec:error} by providing non-asymptotic high probability error bound for parameter recovery.
We delineate our geometrically convergent algorithm, \dc{} in Section \ref{sec:opt} and finally supplement our work with experiments on synthetic data in Section \ref{sec:expds}.

\begin{figure}[h]
	\centering
	\begin{subfigure}[t]{0.2\textwidth}
		\includegraphics[width=\textwidth]{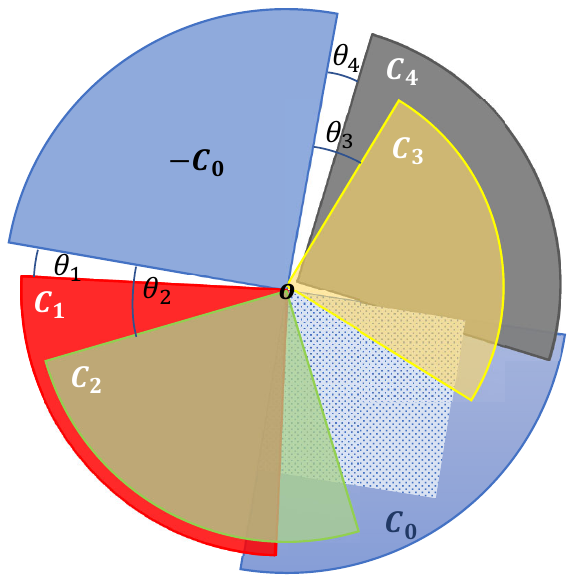}
		\caption{Structural Coherence (SC) condition.}\label{fig:sc}
	\end{subfigure} 
	~~~~~~~~
	\begin{subfigure}[t]{0.2\textwidth}
		\includegraphics[width=\textwidth]{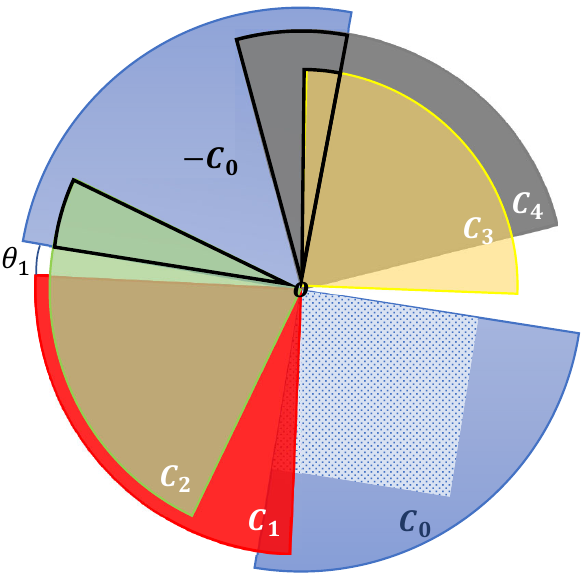}
		\caption{DAta SHaring Incoherence conditioN (\ds).}
		\label{fig:DASHIN}
	\end{subfigure}
	\caption{\small Comparison of geometric recovery condition for superposition models known as Structural Coherence (SC) \cite{guba16} and our \ds\ recovery condition for data sharing model which is a system of coupled superposition models \eqref{eq:dirtymodel}. For each parameter $\bbeta^*_g \in [G]$, $\cE_g = \left\{\ddelta_g | f_g(\bbeta _g^* + \ddelta_g) \leq f_g(\bbeta _g^*)\right\}$ is the error set and $\cC_g = \text{Cone}(\cE_g)$ is the error cone. For all $i,j$, SC requires $-\cC_i \cap \cC_j = \{0\}$. In panel (a) we only show this condition for $i = 0$, i.e., $-\cC_0 \cap \cC_j = \{0\}$ where all $\theta_j > 0$. (b) \ds\ only needs one of the $\cC_g, g \in [G],$ does not intersect with the inverse of the common parameter error cone $-\cC_0$. In panel (b) $-\cC_0 \cap \cC_1 = \{0\}$ is enough for recovering all parameters.}
	\label{fig syn2}
\end{figure}

\section{The Data Shared Estimator}
\label{sec:esti}
A compact form of our proposed DS estimator \eqref{eq:super} is:
{\small \beq
\label{eq:compact}
\hbbe \in \argmin_{\bbeta } \frac{1}{n} \norm{\y - \X \bbeta }{2}^2, \text{s.t. } \forall g \in [G_+]:f_g(\bbeta_g ) \leq f_g(\bbeta^*_g),
\eeq }
where $\y  = (\y^T_1, \dots \y^T_G)^T \in \reals^n$,  $\bbeta  = ({\bbeta _0}^T, \dots, {\bbeta _G}^T)^T \in \reals^{(G+1)p}$ and
\beq
\label{eq:x}
\X =
\begin{pmatrix}
	\X_1     & \X_1      & 0      	   & \cdots & 0 \\
	\X_2     & 0       	 & \X_2        & \cdots & 0 \\
	\vdots 	 & \vdots  	 & \ddots 	   & \cdots & \vdots  \\
	\X_G     & 0       	 & \cdots 	   & \cdots & \X_G
\end{pmatrix}
\in \reals^{n \times (G+1)p}~.
\eeq

\begin{example} \label{exm:sde}
{\bf ($l_1$-norm)} When all parameters $\bbeta_g$s are $s_g$-sparse, i.e.,$|\text{supp}(\bbeta_{g}^*)| = s_g$ by using $l_1$-norm as the sparsity inducing function, our DS estimator of \eqref{eq:compact} becomes:
\be \nr 
\hbbe \in \argmin_{\bbeta } \frac{1}{n} \norm{\y - \X \bbeta }{2}^2, \text{s.t. }  \forall g \in [G_+]: \norm{\bbeta_g}{1} \leq \norm{\bbeta^*_g}{1}.
\ee 
We call \eqref{exm:sde} \emph{sparse DS} estimator and use it as the running example throughout the paper to illustrate outcomes of our analysis.
\end{example}
Consider the group-wise estimation error $\ddelta_g = \hbbe_g - \bbeta^*_g$.
Since $\hbbe_g = \bbeta ^*_g + \ddelta_g$ is a feasible point of \eqref{eq:compact}, the error vector $\ddelta_g$ will belong to the following restricted error set:
\be
\nr
\cE_g = \left\{\ddelta_g | f_g(\bbeta _g^* + \ddelta_g) \leq f_g(\bbeta _g^*)\right\}, \quad g \in [G_+].
\ee
We denote the cone of the error set as $\cC_g \triangleq \text{Cone}(\cE_g)$ and the spherical cap corresponding to it as $\cA_g \triangleq \cC_g \cap \sphere$.
Consider the set $\cC = \{ \ddelta = (\ddelta_0^T, \dots, \ddelta_G^T)^T \Big| \ddelta_g \in \cC_g \}$, following two subsets of $\cC$ play key roles in our analysis:
\begin{align}
\label{setH}
\cH  &\triangleq  \Big\{ \ddelta \in \cC \big| \sum_{g=0}^{G} {\frac{n_g}{n}} \norm{\ddelta_g}{2} = 1 \Big\}, 
\\ \nr 
\bcH &\triangleq  \Big\{ \ddelta \in \cC \big| \sum_{g=0}^{G} \sqrt{\frac{n_g}{n}} \norm{\ddelta_g}{2} = 1 \Big\}~. 
\end{align}
	
Starting from the optimality of $\hbbe = \bbeta ^* + \ddelta$ as $\frac{1}{n}\norm{\y - \X \hbbe}{2}^2 \leq \frac{1}{n} \norm{\y - \X \bbeta ^*}{2}^2$, we have: $\frac{1}{n}\norm{\X \ddelta}{2}^2 \leq \frac{1}{n}2\w^T \X\ddelta$ where $\w = [\w_1^T, \dots, \w_G^T]^T \in \reals^n$ is the vector of all noises.
Using this basic inequality, we can establish the following deterministic error bound.
\begin{theorem}
	\label{theo:deter}
	For the DS estimator \eqref{eq:compact}, assume there exist $0 < \kappa \leq \inf_{\u \in \cH} \frac{1}{n} \norm{\X \u}{2}^2$. Then, for the sample condition number $\gamma = \max_{g \in [G]} \frac{n}{n_g}$, the following deterministic upper bounds holds:
	\be
	\nr
	\sum_{g=0}^{G} \sqrt{\frac{n_g}{n}} \norm{\ddelta_g}{2} \leq \frac{2{\gamma}\sup_{\u \in \bcH}\w^T \X \u}{n\kappa}~. 
	\ee
\end{theorem}
\begin{IEEEproof}
	We lower bound the LHS and upper bound the RHS of the optimality inequality $\frac{1}{n}\norm{\X \ddelta}{2}^2 \leq \frac{1}{n}2\w^T \X\ddelta$ using the definition of the sets $\cH$ and $\bcH$ respectively. 
	Starting with the lower bound using the definition of set $\cH$ \eqref{setH} we have:
	{\small
	\begin{align}
		\nr 
		\frac{1}{n}\norm{\X \ddelta}{2}^2 &\geq \frac{1}{n} \inf_{\u \in \cH} \norm{\X \u}{2}^2  \left(\sum_{g=0}^{G} {\frac{n_g}{n}} \norm{\ddelta_g}{2} \right)^2 
		\geq \kappa  \left(\sum_{g=0}^{G} {\frac{n_g}{n}} \norm{\ddelta_g}{2} \right)^2  
		\\  \label{eq:tre}  
		&\geq \kappa  \left(\min_{g \in [G] } \sqrt{\frac{n_g}{n}}\right)^2 \left(\sum_{g=0}^{G} \sqrt{\frac{n_g}{n}} \norm{\ddelta_g}{2} \right)^2 
		\\ \nr 
		&= \kappa  \left(\min_{g \in [G] } \frac{n_g}{n}\right) \left(\sum_{g=0}^{G} \sqrt{\frac{n_g}{n}} \norm{\ddelta_g}{2} \right)^2  		 		
	\end{align}
	}

	where $0 < \kappa \leq \frac{1}{n}  \inf_{\u \in \cH} \norm{\X \u}{2}^2 $ is known as Restricted Eigenvalue (RE) condition. 
	The upper bound factorizes as:
	\beq 
	\label{eq:tub}
	\frac{2}{n}\w^T \X\ddelta \leq \frac{2}{n} \sup_{\u \in \bcH} \w^T \X \u \left(\sum_{g=0}^{G} \sqrt{\frac{n_g}{n}} \norm{\ddelta_g}{2} \right) , \u \in \cH \\ 
	\eeq
	Putting together inequalities \eqref{eq:tre} and \eqref{eq:tub} completes the proof. 
\end{IEEEproof}

\begin{remark}
	\label{rem1}
Consider the setting where $n_g = \Theta(\frac{n}{G})$ so that each group has approximately $\frac{1}{G}$ fraction of the samples. Then, $\gamma = \Theta(G)$ and hence
\beq
\nr 
\frac{1}{G} \sum_{g=0}^G \| \delta_g \|_2 \leq O( G^{1/2} ) \frac{\sup_{\u \in \bcH}\w^T \X \u}{n}~.
\eeq
\end{remark}


\section{Restricted Eigenvalue Condition}
\label{sec:re}
The main assumption of Theorem \ref{theo:deter} is known as Restricted Eigenvalue (RE) condition in the literature of high-dimensional statistics \cite{banerjee14, nrwy12, raskutti10}:
$\inf_{\u \in \cH} \frac{1}{n} \norm{\X \u}{2}^2 \geq \kappa > 0.$
The RE condition posits that the minimum eigenvalues of the matrix $\X^T \X$ in directions restricted to $\cH$ is strictly positive.
In this section, we show that for the design matrix $\X$ defined in \eqref{eq:x}, the RE condition holds with high probability under a suitable geometric condition we call {\em DAta SHaring Incoherence conditioN} (\ds) and for enough number of samples.
We precisely characterize total and per-group sample complexities required for successful parameter recovery.
For the analysis, similar to existing work \cite{guba16, mend15, trop15}, we assume the design matrix to be isotropic sub-Gaussian.\footnote{Extension to an-isotropic sub-Gaussian case is straightforward by techniques developed in \cite{banerjee14, ruzh13}.}
\begin{definition}
	\label{def:obs}
	We assume $\x_{gi}$ are i.i.d. random vectors from a non-degenerate zero-mean, isotropic sub-Gaussian distribution. In other words, $\ex [\x] = 0$, $\ex [\x^T \x] = \I_{p \times p}$, and $\normth{\x}{\psi_2} \leq k_x$.	
As a consequence, $\exists \alpha > 0$ such that $\forall \u \in \sphere$ we have $ \ex|\langle \x, \u \rangle| \geq \alpha$. Further, we assume noise $w_{gi} $ are i.i.d.
zero-mean, unit-variance sub-Gaussian with $\normth{w_{gi}}{\psi_2} \leq k_w$.
\end{definition}

\subsection{Geometric Condition for Recovery}
Unlike standard high-dimensional statistical estimation, for RE condition to be true, parameters of superposition models need to satisfy geometric conditions which limits the interaction of the error cones of parameters with each other to make sure that recovery is possible. In this section, we elaborate our sufficient geometric condition for recovery and compare it with state-of-the-art condition for recovery of superposition models. 

To intuitively illustrate the necessity of such a geometric condition, consider the simplest superposition model i.e., $\bbeta^*_0 + \bbeta^*_g$. Without any restriction on interactions of error cones, any estimates such that $\hbbe_0 + \hbbe_g = \bbeta^*_0 + \bbeta^*_g$ are valid ones. To avoid such trivial solutions two error cones need to satisfy $\ddelta_g \neq -\ddelta_0$. In general, the RE condition of individual superposition models can be established under the so-called Structural Coherence (SC) condition \cite{guba16, mctr13} which is the generalization of this idea for superposition of multiple parameters as $\sum_{g = 0}^{G} \bbeta^*_g$.
 
\begin{definition}[Structural Coherence (SC) \cite{guba16, mctr13}] \label{scc}
	Consider a superposition model of the form $y = \x^T \sum_{g = 0}^{G} \bbeta^*_g + w$. The SC condition requires that
$	\forall \ddelta_g \in \cC_g, \exists \lambda \text{ s.t. }  \norm{\sum_{g = 0}^{G} \ddelta_g}{2} \geq  \lambda \sum_{g = 0}^{G}  \norm{\ddelta_g}{2}$,
	and leads to the RE condition $\frac{1}{\sqrt{n}}\norm{\X \sum_{g = 1}^{G} \ddelta_g}{2} \geq \kappa \sum_{g=1}^{G} \norm{\ddelta_g}{2}$.
\end{definition}

\begin{remark}
	Note that the SC condition is satisfied if none of the individual error cones $\cC_g$ intersect with the inverted error cone $-\cC_0$ \cite{guba16, trop15}, i.e., $\forall g, \theta_g > 0$ in Fig. \ref{fig:sc} where 
	\be 
	\nr 
	\cos(\theta_g) = \sup_{\ddelta_0 \in \cC_0, \ddelta_g \in \cC_g} -\langle \ddelta_0/\norm{\ddelta_0}{2}, \ddelta_g/\norm{\ddelta_g}{2} \rangle.
	\ee
\end{remark}
Next, we introduce \ds, a considerably weaker geometric condition compared to SC which leads to recovery of all parameters in the data sharing model. 
\begin{definition}[DAta SHaring Incoherence conditioN (\ds)]  \label{incodef}
	There exists a non-empty set $\cI\subseteq [G]$ of groups where for some scalars $0 < \ratio\leq 1$ and $\lamin>0$ the following holds:
	\begin{enumerate}
		\item $\sum_{g\in \cI} n_g\geq \lceil \ratio n\rceil$.
		\item $\forall g \in \cI$, $\forall \ddelta_g \in \cC_g$, and $\ddelta_0\in\cC_0$: $\norm{\ddelta_g+\ddelta_0}{2}\geq \lamin (\norm{\ddelta_0}{2}+\norm{\ddelta_g}{2})$
	\end{enumerate}
	Observe that $0 < \lamin,\ratio\leq 1$ by definition.
\end{definition}



\begin{remark}
	\ds\ is a refinement of SC for the specific problem of data sharing, i.e., system of coupled superposition model each with two components. \ds\ holds even if only one of the $\cC_g$s does not intersect with $-\cC_0$. More specifically, \ds\ holds if $\exists g, \theta_g > 0$ in Fig. \ref{fig:DASHIN} which allows $-\cC_0$ to intersect with an arbitrarily large fraction of the $\cC_g$ cones and as the number of intersections increases, our final error bound becomes looser.
\end{remark}

\subsection{Sample Complexity}
An alternative to our DS estimator \eqref{eq:super} may be based on $G$ \emph{isolated} superposition model $\y_g = \X_g (\bbeta _0^* + \bbeta _g^*) + \w_g$ each with two components. Now, if SC holds for at least one of the superposition models, i.e., $\exists g, -\cC_0 \cap \cC_g = \{0 \}$, one can recover $\hbbe_0$ and plug it in to the remaining $G-1$ superposition estimators to estimate the corresponding $\hbbe_g$s. We call such an estimator, \emph{plugin superposition} estimator for which it seems that \ds\ has no advantage over SC. But the disadvantage of plugin superposition estimator is that it fails to utilize the true coupling structure in the data sharing model, where $\bbeta^*_0$ is involved in all groups. In fact, below we show that the plugin superposition estimator under SC condition leads to a pessimistic sample complexity for $\bbeta^*_0$ recovery.
\begin{prop}
	\label{prop:super}
	Assume observations distributed as defined in \ref{def:obs} and pair-wise SC conditions are satisfied.  Consider each superposition model \eqref{eq:dirtymodel} in isolation; to recover the common parameter $\bbeta _0^*$ plugin superposition requires at least one group $i$ to have $n_i = O(\max(\omega^2(\cA_0), \omega^2(\cA_i)))$. 
	To recover the rest of parameters, it needs $\forall g \neq i: n_g = O(\omega^2(\cA_g))$ samples. 
\end{prop}
In other words, by separate analysis of superposition estimators at least one problem needs to have sufficient samples for recovering the common parameter $\bbeta_0$ and therefore the common parameter recovery does not benefit from the pooled $n$ samples.
But given the nature of coupling in the data sharing model, we hope to be able to get a better sample complexity specifically for the common parameter $\bbeta_0$.
Using \ds\ and the small ball method \cite{mend15}, a tool from empirical process theory in the following theorem, we get a better sample complexity required for satisfying the RE condition:
\begin{theorem}
	\label{theo:re}
	Let $\x_{gi}$s	be random vectors defined in Definition \ref{def:obs}.
	Assume \ds\ condition of Definition \ref{incodef} holds for error cones $\cC_g$s and $\rinc=\min\{1/2, \lamin\ratio/3\}$.
	Then, for all $\ddelta \in \cH$, when we have enough number of samples as $\forall g \in [G_+]: n_g \geq m_g = O(k_x^6 \alpha^{-6} \rinc^{-2} \omega(\cA_g)^2)$, with probability at least $1 - e^{-n \kappa_{\min}/4}$  we have:
	$\inf_{\ddelta \in \cH} \frac{1}{\sqrt{n}} \norm{\X \ddelta}{2} \geq \frac{\kappa_{\min}}{2}$,
	where $\kappa_{\min} = \min_{g\in [G_+]} C \rinc \frac{\alpha^3}{k_x^2}  - \frac{2 c_g k_x \omega(\cA_g)}{\sqrt{n_g}}$. 
\end{theorem}

\begin{remark}
	Note that $\kappa = \frac{\kappa_{\min}^2}{4}$ is the lower bound of the RE condition of Theorem \ref{theo:deter}, i.e., $0 < \kappa \leq \inf_{\u \in \cH} \frac{1}{n} \norm{\X \u}{2}^2$ and is determined by the group with the worst RE condition. 
\end{remark}

\begin{example}
	{\bf ($l_1$-norm)} The Gaussian width of the spherical cap of a $p$-dimensional $s$-sparse vector is $\omega(\cA) = \Theta(\sqrt{s\log p})$ \cite{banerjee14, vershynin2018high}. Therefore, the number of samples per group and total required for satisfaction of the RE condition in the sparse DS estimator Example \ref{exm:sde} is $\forall g \in [G]: n_g \geq m_g = \Theta(s_g \log p)$. 
	Table \ref{compare} compares sample complexities of the sparse DS estimator with three baselines: plugin superposition estimator of Proposition \ref{prop:super}, G Independent LASSO (GI-LASSO), and Jalali's Dirty Statistical Model (DSM) \cite{jrsr10}. Note that GI-LASSO does not recover the common parameter and DSM needs all groups have same number of samples. 
\end{example}
\begin{table*}[]
	\centering
	\begin{tabular}{l|c|c|c|c|}
		\cline{2-5}
		& \textbf{GI-LASSO} & \textbf{Dirty Stat. Model}                      & \textbf{Plugin Superposition}                                                                                                           & \textbf{Sparse DS} \\ \hline
		\multicolumn{1}{|l|}{\textbf{$m_g$}} & $s_{0g} \log p$   & $G \max_{g \in [G]} s_{0g} \log(p)$ & \begin{tabular}[c]{@{}c@{}}$\exists i \in [G]: \max(s_0, s_i) \log p$ \\ $\forall g \neq i: s_{g} \log p$\end{tabular} & $s_{g} \log p$     \\ \hline
	\end{tabular}
	\caption{\small Comparison of the order of per group number of samples (sample complexities) of various methods for recovering sparse DS parameters. Let $s_{0g} = |\text{support}(\bbeta^*_0 + \bbeta^*_g)|$ be the superimposed support where $s_0, s_g \leq \max(s_0, s_g) \leq s_{0g}$.}
	\label{compare}
\end{table*}


\subsection{Proof of Theorem \ref{theo:re}}
Let's simplify the LHS of the RE condition:
{\small\begin{align}
\nr 
&\frac{1}{\sqrt{n}} \norm{\X \ddelta}{2} 
= \left(\frac{1}{n} \sum_{g=1}^{G} \sum_{i=1}^{n_g} |\langle \x_{gi}, \ddelta_0 + \ddelta_g \rangle|^2\right)^{\frac{1}{2}}
\\ \nr
&\geq \frac{1}{n} \sum_{g=1}^{G} \sum_{i=1}^{n_g} |\langle \x_{gi}, \ddelta_0 + \ddelta_g \rangle| 
\\ \nr 
&\geq \frac{1}{{n}} \sum_{g=1}^{G} \xi\norm{\ddelta_0+\ddelta_g}{2}  \sum_{i=1}^{n_g} \indic \left(|\langle \x_{gi}, \ddelta_0 + \ddelta_g \rangle| \geq \xi\norm{\ddelta_0+\ddelta_g}{2}\right)
\\ \nr 
&= \frac{1}{{n}} \sum_{g=1}^{G} \xi_g  \sum_{i=1}^{n_g} \indic \left(|\langle \x_{gi}, \ddelta_{0g} \rangle| \geq \xi_g\right),
\end{align}}
where to avoid cluttering we denoted $\ddelta_{0g} = \ddelta_0 + \ddelta_g$ and $\xi_g = \xi \norm{\ddelta_{0g}}{2} > 0$.
Now we add and subtract the corresponding per-group marginal tail function, $Q_{\xi_g}(\ddelta_{0g}) = \pr(|\langle \x, , \ddelta_{0g} \rangle| > \xi_g)$ and take $\inf$: 
\begin{align}
\nr
&\inf_{\ddelta \in \cH} \frac{1}{\sqrt{n}} \norm{\X \ddelta}{2}
\geq \inf_{\ddelta\in \cH}\sum_{g=1}^{G}  \frac{n_g}{n}  \xi_g  Q_{2 \xi_g}(\ddelta_{0g}) 
\\ \nr 
&-	\sup_{\ddelta\in \cH} \frac{1}{n} \sum_{g=1}^{G}  \xi_g  \sum_{i=1}^{n_g} \left[Q_{2 \xi_g}(\ddelta_{0g})  
- \indic (|\langle \x_{gi}, \ddelta_{0g} \rangle| \geq   \xi_g)  \right]
\\ \label{eq:long}
&= t_1(\X)-t_2(\X) 
\end{align}
For the ease of exposition we consider the LHS of \eqref{eq:long}  as the difference of two terms, i.e., $t_1(\X) - t_2(\X)$ and in the followings we lower bound $t_1$ and upper bound $t_2$. 

\subsubsection{Lower Bounding the First Term $t_1(\X)$}
First, note that $t_1(\X)$ is the weighted summation of $\xi_g Q_{2\xi_g}(\ddelta_{0g}) = \norm{\ddelta_{0g}}{2}\xi \pr(|\langle \x, , \ddelta_{0g}/\norm{\ddelta_{0g}}{2} \rangle| > 2\xi) = \norm{\ddelta_{0g}}{2}\xi Q_{2\xi}(\u)$ where $\xi > 0$ and $\u = \ddelta_{0g}/\norm{\ddelta_{0g}}{2}$ is a unit length vector. Using the Paley--Zygmund inequality for the sub-Gaussian random vector $\x$ \cite{trop15}, we have $Q_{2\xi}(\u)  \geq q \triangleq \frac{(\alpha - 2\xi)^2}{4ck_x^2}$, where $q$ is a constant. Therefore, $t_1(\X) \geq \xi q \sum_{g=1}^G\frac{n_g}{n}\norm{\ddelta_{0} +  \ddelta_{g}}{2}$.Below lemma provides a lower bound for the remaining summation.
\begin{lemma} \label{incolem main} Suppose that the \ds\ condition of Definition \ref{incodef} holds. Then, for error vectors $g \in [G_+]: \ddelta_g \in \cC_g$, we have: 
	\be \nr 
	\sum_{g=1}^G \frac{n_g}{n} \norm{\ddelta_0+\ddelta_g}{2}\geq\frac{\ratio\lamin}{3} \left( G\norm{\ddelta_0}{2} + \sum_{g=1}^G \frac{n_g}{n} \norm{\ddelta_g}{2} \right).
	\ee 
\end{lemma}	
\begin{IEEEproof}
	We split $[G]-\cI$ into two groups $\cJ,\cK$. $\cJ$ consists of $\ddelta_g$'s with $\norm{\ddelta_g}{2}\geq 2\norm{\ddelta_0}{2}$ and $\cK=[G]-\cI-\cJ$. We use the bounds
	\[
	\norm{\ddelta_0+\ddelta_g}{2}\geq 
	\begin{cases}
	\lamin(\norm{\ddelta_g}{2}+\norm{\ddelta_0}{2}) &\text{if}~g\in \cI
	\\ 
	\norm{\ddelta_g}{2}/2 &\text{if}~g\in \cJ
	\\
	0 &\text{if}~g\in \cK			
	\end{cases}
	\] 
	This implies
	\[
	\sum_{g=1}^G n_g\norm{\ddelta_0+\ddelta_g}{2}\geq \sum_{g\in \cJ}\frac{n_g}{2}\norm{\ddelta_g}{2}+\lamin\sum_{g\in \cI} n_g (\norm{\ddelta_g}{2}+\norm{\ddelta_0}{2}).
	\]
	Let $S_\cS=\sum_{g\in \cS}n_g\norm{\ddelta_g}{2}$ for $\cS=\cI,\cJ,\cK$.
	We know that over $\cK$, $\norm{\ddelta_g}{2}\leq 2\norm{\ddelta_0}{2}$ which implies $S_\cK = \sum_{g\in \cK}n_g\norm{\ddelta_g}{2}\leq 2\sum_{g\in \cK}n_g\norm{\ddelta_0}{2}\leq 2n\norm{\ddelta_0}{2}$. Set $\rinc=\min\{1/2,\lamin\ratio/3\}$. 
	Using $1/2\geq \rinc$, we write:
	\begin{align}
	\nr 
	&\sum_{g=1}^G n_g\norm{\ddelta_0+\ddelta_g}{2}
	\geq \rinc S_\cJ +\lamin\sum_{g\in \cI}n_g (\norm{\ddelta_g}{2}+\norm{\ddelta_0}{2})
	\\ \nr 
	&\geq \rinc S_\cJ +\rinc S_\cK - 2\rinc n\norm{\ddelta_0}{2}+(\sum_{g\in \cI} n_g)\lamin \norm{\ddelta_0}{2}+\lamin S_{\Ic}
	\\ \nr 
	&\geq \rinc (S_\cI + S_\cJ + S_\cK)+ ((\sum_{g\in \cI} n_g)\lamin-2\rinc n)\norm{\ddelta_0}{2}.
	\end{align} 
	The first assumption of Definition \eqref{incodef}, $\sum_{g\in \cI} n_g \geq \ratio n$ implies:
	\be 
	\nr 
	(\sum_{g\in \cI} n_g)\lamin-2\rinc n\geq (\ratio\lamin -2\rinc)n\geq \rinc n.
	\ee 
	Combining all, we obtain:
	\begin{align}
	\nr 
	\sum_{g=1}^Gn_g \norm{\ddelta_0+\ddelta_g}{2} &\geq \rinc (S_\cI + S_\cJ + S_\cK + \norm{\ddelta_0}{2}) 
	\\ \nr 
	&= \rinc(n\norm{\ddelta_0}{2} +\sum_{g=1}^G n_g\norm{\ddelta_g}{2}).
	\end{align}
\end{IEEEproof}
	

\subsubsection{Upper Bounding the Second Term $t_2(\X)$}
First we show $t_2(\X)$ satisfies the bounded difference property defined in Section 3.2. of \cite{boucheron13}, i.e., by changing each of $\x_{gi}$ the value of $t_2(\X)$ at most change by one. 
We rewrite $t_2$ as $t_2(\X) = \sup_{\ddelta \in \cH} g_\delta(\X)$ where $g_\delta\left(\X \right)$ is the argument of $\sup$ in \eqref{eq:long}.
Now we denote the design matrix resulted from replacement of $k$th sample from $j$th group $\x_{jk}$ with another sample $\x'_{jk}$ by $\X'_{jk}$. Then our goal is to show $\forall j \in [G], k \in [n_j], \sup_{\X, \x_{jk}'} |t_2\left(\X \right)  - t_2(\X'_{jk} )|  \leq c_i$ for some constant $c_i$. 
Note that for bounded functions $f, g: \cX \rightarrow \reals$, we have $|\sup_{\cX} f - \sup_{\cX} g| \leq \sup_{\cX} |f - g|$. 
Therefore:
\begin{align}
\nr 
&\sup_{\X, \x_{jk}'} |t_2\left(\X \right)  - t_2\left(\X'_{jk} \right)|
\leq \sup_{\X, \x_{jk}'} \sup_{\ddelta \in \cH} \big|g\left(\X \right) - g\left(\X'_{jk} \right) \big|
\\ \nr 
&\leq \sup_{\x_{jk},  \x_{jk}'} \sup_{\ddelta \in \cH} \frac{\xi_j}{n} \left|\indic (|\langle \x_{jk}', \ddelta_{0j}\rangle| \geq   \xi_j)  - \indic (|\langle \x_{jk}, \ddelta_{0j} \rangle| \geq   \xi_j) \right| 
\\ \nr 
&\leq \sup_{j} \sup_{\ddelta \in \cH} \frac{\xi_j}{n} 
= \frac{\xi}{n} \sup_{j} \sup_{\ddelta \in \cH} {\norm{\ddelta_0 + \ddelta_j}{2}}
\\ \nr 
&\leq \frac{\xi}{n} \sup_{j} \sup_{\ddelta \in \cH} \norm{\ddelta_0}{2} + \norm{\ddelta_j}{2}
\leq \xi \left(\frac{1}{n} + \frac{1}{{n_j}}\right) 
\leq  \frac{2\xi}{n}
\end{align}
 
Note that for $\ddelta \in \cH$ we have $\norm{\ddelta_0}{2} + \frac{n_g}{n}\norm{\ddelta_g}{2} \leq 1$ which results in $\norm{\ddelta_0}{2} \leq 1$ and $\norm{\ddelta_g}{2} \leq \frac{n}{n_g}$ which justifies the last inequality. 
Now, we can invoke the bounded difference inequality from Theorem 6.2 of \cite{boucheron13} which says that with probability at least $1 - e^{-\tau^2/2}$ we have: $t_2(\X) \leq \ex t_2(\X) + \frac{\tau}{\sqrt{n}}$. 
Having this concentration bound, it is enough to bound the expectation of $t_2(\X)$ using the following lemma:
\begin{lemma}
	\label{lemm:secTerm}
	For the random vector $\x$ of Definition \ref{def:obs}, we have the following bound:
	{\small\begin{align}	
	\nr 
	\ex t_2(\X) &= \frac{2}{n} \ex \sup_{\ddelta \in \cH} \sum_{g=1}^{G} \xi_g \sum_{i=1}^{n_g} \left[Q_{2 \xi_g}(\ddelta_{0g})  - \indic (|\langle \x_{gi}, \ddelta_{0g} \rangle| \geq \xi_g )  \right]
	\\ \nr 
	&\leq \frac{2}{\sqrt{n}} \sum_{g=0}^{G}  \sqrt{\frac{n_g}{n}} c_g k \omega(\cA_g) \norm{\ddelta_{g}}{2}.
	\end{align}}
\end{lemma}
\begin{IEEEproof}
	Following the similar steps of proof of Proposition 5.1 of \cite{trop15}, $\ex t_2(\X)$ can be bounded by $\frac{2}{n} \ex \sup_{\ddelta \in \cH} \sum_{g=1}^{G} \sum_{i=1}^{n_g} \epsilon_{gi} \langle \x_{gi}, \ddelta_{0g} \rangle$ where $\epsilon_{gi}$ are iid copies of Rademacher random variable which are independent of every other random variables and themselves. Now, we expand $\ddelta_{0g} = \ddelta_{0} + \ddelta_{g}$ and define $\h_{g} \triangleq \frac{1}{\sqrt{n_g}} \sum_{i=1}^{n_g} \epsilon_{gi} \x_{gi}$ to simplify the notation. Also, we substitute $\ddelta \in \cH$ constraint with $\ddelta \in \cC$ because $\cH \subseteq \cC$. We have:
	\begin{align} \nr 
		\ex t_2(\X) &\leq \frac{2}{\sqrt{n}} \ex \sup_{\forall g \in [G_+]: \ddelta_g \in \cC_g} \sum_{g=0}^{G}  \sqrt{\frac{n_g}{n}}  \langle \h_{g}, \ddelta_{g} \rangle
		\\ \nr 
		&\leq \frac{2}{\sqrt{n}} \ex \sup_{\forall g \in [G_+]: \u_g \in \cA_{g}} \sum_{g=0}^{G}  \sqrt{\frac{n_g}{n}} \langle \h_{g}, \u_{g} \rangle \norm{\ddelta_{g}}{2}
		\\ \nr 
		&\leq \frac{2}{\sqrt{n}} \sum_{g=0}^{G}  \sqrt{\frac{n_g}{n}} \ex_{\h_{g}} \sup_{\u_g \in \cA_g}  \langle \h_{g}, \u_{g} \rangle \norm{\ddelta_{g}}{2}
		\\ \nr 
		&\leq \frac{2}{\sqrt{n}} \sum_{g=0}^{G}  \sqrt{\frac{n_g}{n}} c_g k_x \omega(\cA_g) \norm{\ddelta_{g}}{2}.
	\end{align}
	Note that the $\h_{gi}$ is a sub-Gaussian random vector which let us bound the $\ex \sup$ using the Gaussian width \cite{trop15} in the last step. 
\end{IEEEproof}

\subsubsection{Continuing the Proof of Theorem \ref{theo:re}}
Putting back bounds of $t_1(\X)$ and $t_2(\X)$ together from Lemmas \ref{incolem main} and \ref{lemm:secTerm}, with probability at least $1 - e^{-\frac{\tau^2}{2}}$ we have:
{\small
\begin{align}
\nr 
&\inf_{\ddelta \in \cH} \frac{1}{\sqrt{n}} \norm{\X \ddelta}{2}
\\ \nr
&\leq\sum_{g=0}^{G}  \frac{n_g}{n} \rinc \xi \norm{\ddelta_g}{2} q
- \frac{2}{\sqrt{n}} \sum_{g=0}^{G}  \sqrt{\frac{n_g}{n}} k_x c_g \omega(\cA_g) \norm{\ddelta_{g}}{2} - \frac{\tau }{\sqrt{n}}
\\ \nr
&=n^{-1}\sum_{g=0}^{G} n_g \norm{\ddelta_{g}}{2} ( \rinc \xi  q-2 c_g k_x \frac{\omega(\cA_g)}{\sqrt{n_g}})-\frac{\tau}{\sqrt{n}}
\\ \nr
(\text{i}) &= \sum_{g=0}^{G} \frac{n_g}{n} \norm{\ddelta_g}{2} \kappa_g  - \frac{\tau}{\sqrt{n}}
\\ \nr
(\text{ii}) &\geq \kappa_{\min}\sum_{g=0}^{G} \frac{n_g}{n} \norm{\ddelta_g}{2}  - \frac{\tau}{\sqrt{n}}
= \kappa_{\min}  - \frac{\tau}{\sqrt{n}} 
\end{align}
}where $\kappa_g \triangleq \rinc \xi q  - \frac{2 c_g k_x \omega(\cA_g)}{\sqrt{n_g}}$ and $\kappa_{\min} \triangleq \min_{g\in [G]} \kappa_g$ in steps (i) and (ii), and the last step  follows from the fact that $\ddelta \in \cH$ . To conclude the proof, take $\tau = \sqrt{n} \kappa_{\min}/2$. 

To satisfy the RE condition all $\kappa_g$s should be bounded away from zero.
To this end we need the following sample complexities $\forall g \in [G_+]: \left(\frac{2 c_g k }{\rinc \xi q}\right)^2 \omega(\cA_g)^2 \leq n_g $ where by taking $\xi = \frac{\alpha}{6}$ simplifies to: $\forall g \in [G_+]: O\left(k^6 \rinc^{-2} \alpha^{-6} \omega(\cA_g)^2\right) \leq n_g$. {\qedsymbol}

\section{General Error Bound}
\label{sec:error}
In this section, we present our main statistical result which is a non-asymptotic high probability upper bound for the estimation error of the common and individual parameters.
\begin{theorem}
	\label{theo:calcub}
	For $\x_{gi}$ and $w_{gi}$ described in Definition \ref{def:obs} when we have enough number of samples $\forall g \in [G_+]: n_g > m_g$ which lead to $\kappa > 0$, the following general error bound holds for estimator \eqref{eq:compact} with probability at least $1 - \sigma \exp\left(-\min\left[\nu  \min_{g \in [G]} n_g - \log (G+1), \tau^2\right]\right) $: 
	{\small\be
	\label{eq:general}
	\sum_{g=0}^{G} \sqrt{\frac{n_g}{n}} \norm{\ddelta_g}{2}
	\leq C {\gamma} \frac{\max_{g \in [G_+]}  \omega(\cA_g) + \sqrt{\log (G+1)}+ \tau }{\kappa_{\min}^2 \sqrt{n}}
	\ee}where $\gamma = \max_{g \in [G]} n/n_g$, $\tau > 0$, and $\sigma, \nu,$ and $C$ are constants. 
\end{theorem}

\begin{corollary}
	\label{corr:single}
	From \eqref{eq:general} one can immediately entail the error bound for estimation of all parameters as follows:
	{\small\be
	\nr
	\forall g \in [G_+]: \quad \norm{\ddelta_g}{2} =  O\left(\gamma \frac{\max_{g \in [G_+]}  \omega(\cA_g) + \sqrt{\log (G+1)} }{\sqrt{n_g}}\right)
	\ee}
\end{corollary}

\begin{example}
	For the balanced sample condition number $\gamma = \Theta(G)$ discussed in Remark \ref{rem1} we have the following error bound for all parameters:
	{\small\be 
		\forall g \in [G_+]: \norm{\ddelta_g}{2} =  O\left(G^{3/2} \frac{\max_{g \in [G_+]}  \omega(\cA_g) + \sqrt{\log (G+1)} }{\sqrt{n}}\right)
	\ee} 
	where the upper bound of error scales as $\frac{1}{\sqrt{n}}$ for all parameters. 
\end{example}

\begin{example}
	{\bf ($l_1$-norm)} For the sparse DS estimator of Example \ref{exm:sde}, results of Theorems \ref{theo:re} and \ref{theo:calcub} translates to the following. For enough samples as $\forall g \in [G_+]: n_g \geq m_g = O(s_g \log p)$, the upper bound of error simplifies to:
	\be \nr 
	\sum_{g=0}^{G} \sqrt{\frac{n_g}{n}} \norm{\ddelta_g}{2}	= O \left(\gamma \sqrt{\frac{(\max_{g \in [G_+]}  s_g)\log p}{n}}\right) ,
	\ee 
	Therefore, individual errors are bounded as $\norm{\ddelta_g}{2}	= O (\gamma \sqrt{(\max_{g \in [G]}  s_g)\log p/n_g})$
	which is slightly worse than $O(\sqrt{s_g\log p/n_g})$, the well-known error bound for recovering an $s_g$-sparse vector from $n_g$ observations using LASSO or similar estimators \cite{banerjee14, bickel2009simultaneous, candes2007dantzig, venkat12, chatterjee2014generalized}. 
\end{example}

\subsection{Proof of Theorem \ref{theo:calcub}}
To avoid cluttering the notation, we rename the vector of all noises as $\w_0 \triangleq \w$.
First, we massage the deterministic upper bound of Theorem \ref{theo:deter} as follows:
\begin{align}
	\nr
	\w ^T \X\ddelta &= \sum_{g=0}^{G} \langle \X_g^T \w_g,  \ddelta_g \rangle
	\\ \nr 
	&= \sum_{g=0}^{G} \sqrt{\frac{n_g}{n}} \norm{\ddelta_g}{2} \langle \X_g^T \frac{\w_g}{\norm{\w_g}{2}}, \frac{\ddelta_g}{\norm{\ddelta_g}{2}} \rangle \sqrt{\frac{n}{n_g}} \norm{\w_g}{2} 
\end{align}

Assume $q_g = \langle \X_g^T \frac{\w_g}{\norm{\w_g}{2}}, \frac{\ddelta_g}{\norm{\ddelta_g}{2}}  \rangle \sqrt{\frac{n}{n_g}} \norm{\w_g}{2}$ and $p_g = \sqrt{\frac{n_g}{n}} \norm{\ddelta_g}{2}$.
Then the above term is the inner product of two vectors $\p = (p_0, \dots, p_G)$ and $\q = (q_0, \dots, q_G)$ for which we have:
$\sup_{\p \in \bcH} \p^T \q
=\sup_{\norm{\p}{1} = 1} \p^T \q
\leq \norm{\q}{\infty}
= \max_{g \in [G_+]} q_g,
$
where the inequality holds because of the definition of the dual norm.
Going back to the original form:
\bea 
\label{eq:maxex}
\sup_{\ddelta \in \cH}\w^T \X\ddelta
\leq& \max_{g \in [G]} \langle \X_g^T \frac{\w_g}{\norm{\w_g}{2}}, \frac{\ddelta_g}{\norm{\ddelta_g}{2}}  \rangle \sqrt{\frac{n}{n_g}} \norm{\w_g}{2} \\ 
\nr 
\leq& \max_{g \in [G]} \sqrt{\frac{n}{n_g}} \norm{\w_g}{2} \sup_{\u_g \in \cC_g \cap \sphere} \langle \X_g^T \frac{\w_g}{\norm{\w_g}{2}}, \u_g \rangle 
\eea

To avoid cluttering we define a random quantity $h_g(\w_g, \X_g)$ and a corresponding constant $e_g(\tau)$ as: 
{\small\begin{itemize}
	\item $h_g(\w_g, \X_g) \triangleq   \norm{\w_g}{2}  \sup_{\u_g \in \cA_g} \langle \X_g^T \frac{\w_g}{\norm{\w_g}{2}}, \u_g \rangle $
	\item $e_g(\tau) \triangleq  c_g\sqrt{(2k_w^2 + 1)k_x^2n_g} \left(\omega(\cA_g) + \sqrt{\log (G+1)} + \tau \right)$
\end{itemize}}
Then from \eqref{eq:maxex}, we have:
\begin{align}
\nr  
&\pr \left(\sup_{\ddelta \in \cH} \w^T \X\ddelta >  \max_{g \in [G]} \sqrt{\frac{n}{n_g}} e_g(\tau) \right) 
\\ \nr 
&\hspace*{1.5cm}\leq \pr \left(\max_{g \in [G]} \sqrt{\frac{n}{n_g}} h_g(\w_g, \X_g) > \max_{g \in [G]} \sqrt{\frac{n}{n_g}} e_g(\tau) \right) 
\\  \nr 
&\hspace*{1.5cm}\leq \sum_{g=0}^{G} \pr \left(\sqrt{\frac{n}{n_g}} h_g(\w_g, \X_g) >  \max_{g \in [G]}  \sqrt{\frac{n}{n_g}} e_g(\tau) \right)  
\\ \nr 
&\hspace*{1.5cm}\leq \sum_{g=0}^{G} \pr \left( h_g(\w_g, \X_g) >  e_g(\tau) \right)  
\\ \nr 	
&\hspace*{1.5cm}\leq (G+1) \max_{g \in [G_+]} \pr \left(h_g(\w_g, \X_g) > e_g(\tau) \right) 
\\ \nr 
&\hspace*{1.5cm}\leq \sigma \exp\left(-\min\left[\nu  \min_{g \in [G]} n_g - \log (G+1), \tau^2\right]\right), 
\end{align} 
where the first inequality follows from the Union Bound and the last one is the result of the following lemma:
\begin{lemma}[Theorem 4 of \cite{banerjee14}]
	\label{lemm:mainlem}
	For $\x_{gi}$ and $w_{gi}$ defined in Definition \ref{def:obs} and $\tau > 0$, with probability at least $1 - \frac{\sigma_g}{(G+1)} \exp\left(-\min\left[\nu  n_g - \log (G+1), \tau^2\right]\right) $ we have:
	\begin{align}	\nr 
	&\norm{\w_g}{2} \sup_{\u_g \in \cA_g} \langle \X_g^T \frac{\w_g}{\norm{\w_g}{2}}, \u_g \rangle 
	\\ \nr 
	&\hspace*{1cm}\leq 	c_g \sqrt{(2k_w^2 + 1)k_x^2n_g}  \left(\omega(\cA_g)+\sqrt{\log (G+1)} + \tau \right), \nr
	\end{align}
	where $\sigma_g, \nu$ and $c_g$ are constants.
\end{lemma}	
The proof of Theorem \ref{theo:calcub} completes by replacing $\max_{g \in [G]} \sqrt{\frac{n}{n_g}} e_g(\tau)$ as the upper bound of $\sup_{\ddelta \in \cH} \w^T \X\ddelta$ and $\kappa^2_{\min}/4$ as the lower bound of $\kappa$ (from Theorem \ref{theo:re}) both into the bound of Theorem \ref{theo:deter} . \hfill {\qedsymbol}

\section{Estimation Algorithm}
\label{sec:opt}
We propose \emph{DAta SHarER} (\dc) a projected block gradient descent algorithm, Algorithm \ref{alg2}, where $\Pi_{\Omega_{f_g}}$ is the Euclidean projection onto the set $\Omega_{f_g}(d_g) = \{f_g(\bbeta) \leq d_g\}$ where $d_g = f_g(\bbeta_g^*)$ and is dropped to avoid cluttering. 

\begin{algorithm}[t]
	\caption{  \dc }
	\label{alg2}
	\begin{algorithmic}[1]
		\STATE {\bfseries input:} $\X, \y$, learning rates $(\mu_0, \dots, \mu_G)$, initialization $\bbeta ^{(1)} = \0$
		\STATE {\bfseries output:} $\hbbe$
		\FOR{t = 1 \TO T}
		\FOR{g=1 \TO G}
		\STATE {\footnotesize $\bbeta _g^{(t+1)} = \Pi_{\Omega_{f_g}} \left(\bbeta _g^{(t)} + \mu_g \X_g^T \left(\y_g - \X_g \left(\bbeta _0^{(t)} + \bbeta _g^{(t)}\right) \right) \right)$}
		\ENDFOR
		\STATE {\footnotesize $\bbeta _0^{(t+1)} = \Pi_{\Omega_{f_0}} \left(\bbeta _0^{(t)} + \mu_0 \X_0^T \left(\y - \X_0 \bbeta _0^{(t)} -
		\begin{pmatrix}
		\X_1 \bbeta _1^{(t)}      \\
		\vdots 	 \\
		\X_G  \bbeta _G^{(t)}
		\end{pmatrix}\right)\right)$}
		\ENDFOR
	\end{algorithmic}
\end{algorithm}

To analysis convergence properties of \dc, we should upper bound the error of each iteration.
Let's $\ddelta^{(t)} = \bbeta^{(t)} - \bbeta^*$ be the error of  iteration $t$ of \dc, i.e., the distance from the true parameter (not the optimization minimum, $\hbbe$). We show that $\norm{\ddelta^{(t)}}{2}$ decreases exponentially fast in $t$ to the statistical error $\norm{\ddelta}{2} = \norm{\hbbe - \bbeta^*}{2}$. We first start with the required definitions and lemmas for our analysis.

\begin{definition}
	\label{def:only}
	We define the following positive constants as functions of step sizes $\mu_g > 0$: 
	\be
	\nr
	\forall g \in [G_+]&:& \rho_g(\mu_g) = \sup_{\u, \v \in \cB_g} \v^T \big(\I_g - \mu_g \X_g^T \X_g\big) \u,
	\\ \nr
	&&\eta_g(\mu_g) = \mu_g \sup_{\v \in \cB_g} \v^T \X_g^T \frac{\w_g}{\norm{\w_g}{2}},
	\\ \nr
	\forall g \in [G]&:& \phi_g(\mu_g) = \mu_g \sup_{\v \in \cB_g, \u \in \cB_0} -\v^T \X_g^T \X_g \u,
	\ee
	where $\cB_g =  \cC_g \cap \ball$ and $\ball$ is the unit ball.
\end{definition}

\newcommand{\trho}{\tilde{\rho}_g}
\newcommand{\teta}{\tilde{\eta}_g}
\newcommand{\tphi}{\tilde{\phi}_g}

Below lemma shows that these constants are bounded with high probability.  
\begin{lemma}
	\label{lemm:hpub}
	For $\mu_g \geq 0$ the following upper bounds hold:
	\begin{align}		
	\nr 
	&\rho_g\left(\mu_g\right) \leq \trho(\tau) \triangleq \frac{1}{2} \left[1 - \mu_g n_g \left(1 - \sqrt{2} c_g\frac{2 \omega_g + \tau}{ \sqrt{n_g}} \right)\right] , 
	\\ \nr 
	&\hspace*{20pt}\text{with probability at least} 1 - 2\exp\left( -\gamma_g (2\omega(\cA_g) + \tau)^2  \right).
	\\ \nr 
	&\eta_g\left(\mu_g\right) \leq \teta(\tau) \triangleq \mu_g c_g k_x (\omega_g + \tau), 
	\\ \nr 
	&\hspace*{75pt}\text{with probability at least} 1 - \pi_g \exp\left( -\tau^2 \right).
	\\ \nr 
	&\phi_g\left(\mu_g\right) \leq \tphi(\tau) \triangleq \mu_g n_g \left(1 + c_{0g}\frac{\omega_{0g} + \tau}{\sqrt{n_g}} \right), 
	\\ \nr 
	&\hspace*{22pt}\text{with probability at least} 1 - 2\exp\left( -\gamma_g (\omega(\cA_g) + \tau)^2  \right).
	\end{align} 
	where $\omega_g = \omega(\cA_g)$ and $\omega_{0g} = 1/2 [\omega(\cA_g) + \omega(\cA_0)]$ are shorthand and $\trho$, $\teta$, and $\tphi$ are constants determine by $\tau$.
\end{lemma}
\begin{IEEEproof}
	We need the following result from Theorem 11 of \cite{banerjee14}. For the matrix $\X_g$ with independent isotropic sub-Gaussian rows, the following inequalities holds with probability at least $1 - 2\exp\left( -\gamma_g (\omega(\cA_g) + \tau)^2  \right)$ for all $\u_g \in \cC_g$:
	\begin{align} 
	\label{gennips}
	\left(1 -  \alpha_g \right) \norm{\u_g}{2}^2  \leq \frac{1}{n_g}\norm{\X_g\u_g}{2}^2 \leq \left(1 + \alpha_g \right) \norm{\u_g}{2}^2
	\end{align}
	where $\tau > 0$ and $c_g > 0$ are constant and $\alpha_g(\tau) \triangleq c_g\frac{\omega(\cA_g) + \tau}{\sqrt{n_g}}$. 
	Equation \eqref{gennips} characterizes the distortion in the Euclidean distance between points $\u_g \in \cC_g$ when the matrix $\X_g/n_g$ is applied to them and states that any sub-Gaussian design matrix is approximately isometry, with high probability.

	\noindent	\textbf{Bounding $\rho_g(\mu_g)$:}		
		We upper bound the argument of the $\sup$ in $\rho_g(\mu_g)$ definition as follows:	
		\begin{align}
		\nr 
		&\v^T \big(\I_g - \mu_g \X_g^T \X_g\big) \u 
		\\ \nr 
		&=\frac{1}{4}[(\u + \v)^T(\I - \mu_g \X_g^T \X_g) (\u + \v) 		
		\\ \nr 
		&\hspace*{100pt}- (\u - \v)^T(\I - \mu_g \X_g^T \X_g) (\u - \v) ]
		\\ \nr 
		&=\frac{1}{4}[\norm{\u + \v}{2}^2 - \mu_g \norm{\X_g(\u + \v)}{2}^2 
		\\ \nr 
		&\hspace*{110pt}- \norm{\u - \v}{2}^2 + \mu_g \norm{\X_g(\u - \v)}{2}^2 ] \\ \nr 
		&\leq \frac{1}{4}[\left(1 - \mu_g n_g \left[1 -  2\alpha_g(\tau/2)\right]\right) \norm{\u + \v}{2} 
		\\ \nr 
		&\hspace*{81pt}- \left(1 - \mu_g n_g \left[1 +  2 \alpha_g(\tau/2)\right]\right) \norm{\u - \v}{2} ]
		\\ \nr 
		&\leq \frac{1}{4}[\left(1 - \mu_g n_g \right) \left(\norm{\u + \v}{2}  - \norm{\u - \v}{2} \right) 
		\\ \nr 
		&\hspace*{71pt}+   2\mu_g n_g \alpha_g(\tau/2) \left(\norm{\u + \v}{2} + \norm{\u - \v}{2} \right) ]\\ \nr 
		&\leq \frac{1}{2}[\left(1 - \mu_g n_g \right) \norm{\v}{2} +   \mu_g n_g \alpha_g(\tau/2) 2\sqrt{2} ],
		\end{align}		
		where the last line follows from the triangle inequality and the fact that $\norm{\u + \v}{2} + \norm{\u - \v}{2} \leq 2\sqrt{2}$ which itself follows from $\norm{\u + \v}{2}^2 + \norm{\u - \v}{2}^2 \leq 4$.
		Note that we used \eqref{gennips} in the first inequality for bigger sets of $\cA_g + \cA_g$ and $\cA_g - \cA_g$ where Gaussian width of both of them are upper bounded by $2\omega(\cA_g)$, which is contained in $2\alpha_g(\tau/2)$ term.
		
		\noindent \textbf{Bounding $\eta_g(\mu_g)$:}
			The proof of this bound is an intermediate result in the proof of Lemma \ref{lemm:mainlem}.
		
		\noindent \textbf{Bounding $\phi_g(\mu_g)$:}
			The following holds for any $\u$ and $\v$ because of $\norm{\X_g (\u + \v)}{2}^2 \geq 0$:
			\be 
			\nr 
			-\v^T \X_g^T \X_g \u \leq \frac{1}{2} \left(\norm{\X_g \u}{2}^2 + \norm{\X_g \v}{2}^2 \right)
			\ee 
			Therefore, we can bound $\phi_g(\mu_g)$ as follows:	 
			\begin{align}
			\nr  
			\phi_g(\mu_g) &= \mu_g \sup_{\v \in \cB_g, \u \in \cB_0} -\v^T \X_g^T \X_g \u 
			\\ \nr 
			&\leq \frac{\mu_g}{2} n_g \left(\sup_{\u \in \cB_0} \frac{1}{n_g} \norm{\X_g \u}{2}^2 
			+ \sup_{\v \in \cB_g} \frac{1}{n_g} \norm{\X_g \v}{2}^2 \right)
			\\ \nr 
			&\leq \frac{\mu_g n_g}{2} \left(2 + \alpha_0(\tau) + \alpha_g(\tau) \right)
			\\ \nr 
			&\leq \mu_g n_g \left(1 + c_{0g} \frac{1/2 [\omega(\cA_g) + \omega(\cA_0)] + \tau}{\sqrt{n_g}} \right), 
			\end{align}
			where $c_{0g} = \max(c_0, c_g)$. 		
\end{IEEEproof}
Next, we establish a deterministic bound on iteration errors  $\norm{\ddelta_g^{(t)}}{2}$ which depends on constants of Definition \ref{def:only} where to simplify the notation $\mu_g$ arguments are dropped. 
\begin{lemma}
	\label{theo:iter}
	The following deterministic bound for the error at iteration $t + 1$ of Algorithm \ref{alg2}, initialized by $\bbeta ^{(1)} = \0$, holds:
	\begin{align} \label{eq:singleiter}
	&\sum_{g=0}^{G} \sqrt{\frac{n_g}{n}} \norm{\ddelta_g^{(t+1)}}{2}
	\\ \nr 
	&\hspace*{1.2cm}\leq \rho^t \sum_{g=0}^{G}\sqrt{\frac{n_g}{n}}\norm{\bbeta ^*_g}{2}   + \frac{1 - \rho^t}{1 -  \rho}   \sum_{g=0}^{G} \sqrt{\frac{n_g}{n}} \eta_g \norm{\w_g}{2},
	\end{align}
	where $$\rho(\mmu) \triangleq \max\left(\rho_0 + \sum_{g=1}^{G} \sqrt{\frac{n_g}{n}} \phi_g, \max_{g \in [G]} \left[\rho_g + \sqrt{\frac{n}{n_g}}  \frac{\mu_0}{\mu_g} \phi_g \right]  \right),$$ is a constant depending on the vector of step sizes $\mmu = (\mu_0, \dots, \mu_G)$.
\end{lemma}

\begin{IEEEproof}
	First, a similar analysis as that of Theorem 1.2 of \cite{oyrs15} shows that the following recursive dependency holds between the error of $t+1$th and $t$th iterations of \dc{}:
	\begin{align} 
		\nr 
		\norm{\ddelta_g^{(t+1)}}{2} &\leq   \rho_g\norm{\ddelta_g^{(t)}}{2}   +  \xi_g \norm{\oomega_g}{2} + \phi_g \norm{\ddelta_0^{(t)}}{2} 
		\\ \nr 
		\norm{\ddelta_0^{(t+1)}}{2} &\leq   \rho_0 \norm{\ddelta_0^{(t)}}{2} + \xi_0 \norm{\oomega_0}{2} + \mu_0 \sum_{g=1}^{G}  \frac{\phi_g}{\mu_g} \norm{\ddelta_g^{(t)}}{2}  
	\end{align} 
	By recursively applying these inequalities, we get the following deterministic bound:
	
	{\small\begin{align}
	\nr 
	b_{t+1} &= \sum_{g=0}^{G} \sqrt{\frac{n_g}{n}} \norm{\ddelta_g^{(t+1)}}{2} 
	\leq  \left(\rho_0 + \sum_{g=1}^{G} \sqrt{\frac{n_g}{n}} \phi_g\right)  \norm{\ddelta_0^{(t)}}{2} 
	\\ \nr 
	&+ \sum_{g=1}^{G} \left(\sqrt{\frac{n_g}{n}} \rho_g + \mu_0 \frac{\phi_g}{\mu_g} \right) \norm{\ddelta_g^{(t)}}{2} + \sum_{g=0}^{G} \sqrt{\frac{n_g}{n}}  \xi_g \norm{\w_g}{2} 
	\\ \nr
	&\leq  \rho \sum_{g=0}^{G} \sqrt{\frac{n_g}{n}} \norm{\ddelta_g^{(t)}}{2} + \sum_{g=0}^{G} \sqrt{\frac{n_g}{n}}  \xi_g \norm{\w_g}{2}. 
	\\ \nr
	&=  \rho b_{t} +  \sum_{g=0}^{G} \sqrt{\frac{n_g}{n}} \xi_g \norm{\w_g}{2} \\ \nr 
	&\leq \rho^2 b_{t-1}  + ( \rho + 1)  \sum_{g=0}^{G} \sqrt{\frac{n_g}{n}} \xi_g \norm{\w_g}{2} \\ \nr
	&\leq \rho^t b_1  + \left(\sum_{i = 0}^{t-1} \rho^i \right)   \sum_{g=0}^{G} \sqrt{\frac{n_g}{n}} \xi_g \norm{\w_g}{2} \\ \nr 
	&= \rho^t \sum_{g=0}^{G}\sqrt{\frac{n_g}{n}} \norm{\bbeta ^1_g  - \bbeta ^*_g}{2}  + \left(\sum_{i = 0}^{t-1} \rho^i \right)     \sum_{g=0}^{G} \sqrt{\frac{n_g}{n}} \xi_g \norm{\w_g}{2} \\ \nr 
	&\leq \rho^t \sum_{g=0}^{G}\sqrt{\frac{n_g}{n}} \norm{\bbeta ^*_g}{2}   + \frac{1 - \rho^t}{1 -  \rho}   \sum_{g=0}^{G} \sqrt{\frac{n_g}{n}} \xi_g \norm{\w_g}{2} 
	\end{align}	}	
where the last inequality follows from $\bbeta ^1  = \0$.
\end{IEEEproof}

The RHS of \eqref{eq:singleiter} consists of two terms.
If we keep $\rho < 1$, the first term approaches zero fast, and the second term determines the bound. 
In the following, we show that for specific choices of step sizes $\mu_g$s we can keep $\rho < 1$ with high probability and the second term can be upper bounded using the analysis of Section \ref{sec:error}.
More specifically, the first term corresponds to the optimization error which shrinks in every iteration while the second term is of the same order of the upper bound of the statistical error characterized in Theorem \ref{theo:calcub}.

One way for having $\rho < 1$ is to keep all arguments of $\max(\cdots)$  defining $\rho$ strictly below $1$. 
The high probability bounds for constants $\rho_g$, $\eta_g$, and $\phi_g$ provided in Lemma \ref{lemm:hpub} and the deterministic bound of Lemma \ref{theo:iter} leads to the following theorem which shows that for enough number of samples, of the same order as the statistical sample complexity of  Theorem \ref{theo:re}, we can keep $\rho$ below one and have geometric convergence.

\begin{theorem}
	\label{theo:step}		
	Let $\tau = \sqrt{\log(G+1)}/\zeta + \epsilon$ for $\epsilon, \zeta > 0$. For the step sizes:
	\be
	\nr
	\mu_0 \leq \frac{\min_{g \in [G]} h_g(\tau)^{-1}}{2 n} ,
	\forall g \in [G]: \mu_g \leq  \frac{h_g(\tau)^{-1}}{2\sqrt{n n_g}} 
	\ee
	where $h_g(\tau) = \left(1 + c_{0g} \frac{1/2[\omega(\cA_g) + \omega(\cA_0)] + \tau}{\sqrt{n_g}}\right)$
	and sample complexities of $\forall g \in [G_+]: n_g \geq C_g (\omega(\cA_g) + \tau/2)^2$,
	with probability at least $ 1 - \sigma \exp(- \min(\nu \min_{g \in [G]} n_g - \log(G+1), \zeta \epsilon^2) )$ updates of Algorithm \ref{alg2} obey the following:	
	\begin{align}
	\nr
	&\sum_{g=0}^{G} \sqrt{\frac{n_g}{n}} \norm{\ddelta_g^{(t+1)}}{2}
	\leq r(\tau)^t \sum_{g=0}^{G} \sqrt{\frac{n_g}{n}} \norm{\bbeta^*_g}{2}   
	\\ \nr 
	&\hspace*{1.5cm}+ \frac{C(G+1)\sqrt{(2k_w^2 + 1)k_x^2}}{\sqrt{n}(1 - r(\tau))} \left(\max_{g \in [G_+]} \omega(\cA_g) + \tau \right)
	\end{align}
	where 
	{\small$$r(\tau) \triangleq \max\left(\tilde{\rho}_0 + \sum_{g=1}^{G} \sqrt{\frac{n_g}{n}} \tphi, \max_{g \in [G]} \left[\trho + \sqrt{\frac{n}{n_g}}  \frac{\mu_0}{\mu_g} \tphi \right]  \right) < 1,$$} is a constant depending on $\tau$ and $\upsilon, \zeta$, and $\sigma$ are constants.
	
\end{theorem}

\begin{corollary}
	\label{corr:show}
	For enough number of samples, iterations of \dc\ algorithm with step sizes $\mu_0 = \Theta(\frac{1}{n})$ and $\mu_g =  \Theta(\frac{1}{\sqrt{n n_g}})$ geometrically converges to the following with high probability:
	{\small\beq
	\label{eq:scaled}
	\sum_{g=0}^{G} \sqrt{\frac{n_g}{n}} \norm{\ddelta_g^{\infty}}{2}
	\leq c \frac{\max_{g \in [G_+]} \omega(\cA_g) + \sqrt{\log (G+1)}/\zeta +  \theta}{\sqrt{n} (1 - r(\tau))}
	\eeq}
	where $c = C(G+1)\sqrt{(2k_w^2 + 1)k_x^2}$. 
\end{corollary}
	It is instructive to compare RHS of \eqref{eq:scaled} with that of \eqref{eq:general}: $\kappa_{\min}$ defined in Theorem \ref{theo:re} corresponds to $(1 - r(\tau))$ 
	and the extra $G+1$ factor corresponds to the sample condition number $\gamma = \max_{g \in [G] } \frac{n}{n_g}$.
	Therefore, Corollary \ref{corr:show} shows that with the number of samples in the order of sample complexity determined in Theorem \ref{theo:re} \dc{} converges to the statistical error bound determined in Theorem \ref{theo:calcub}.
	
\subsection{Proof Sketch of Theorem \ref{theo:step}}
\label{proofsketch}
	We want to determine $r(\tau)$ such that $\rho(\mmu) < r(\tau) < 1$ with high probability. Here, we provide a proof sketch using the probabilistic bounds on constants $\rho_g$, $\eta_g$, and $\phi_g$ shown in Lemma \ref{lemm:hpub} while leaving out the trivial but tedious computation of the exact high probability provided in Theorem \ref{theo:step}.

	To have $\rho < 1$ in the deterministic bound of Lemma \ref{theo:iter} with the step sizes suggested in Theorem \ref{theo:step}, we need to find the number of samples which satisfy the following conditions:
	
	\begin{itemize}
		\item Condition 1: $\rho_0\left(\mu_0\right) + \sum_{g=1}^{G} \sqrt{\frac{n_g}{n}} \phi_g\left(\mu_g\right) < 1$
		\item Condition 2: $\forall g \in [G]: \rho_g\left(\mu_g\right) + \sqrt{\frac{n}{n_g}} \frac{\mu_0}{\mu_g}\phi_g\left(\mu_g\right) < 1$
	\end{itemize}	
	For Condition 1, from Lemma \ref{lemm:hpub} with high probability, we have the below upper bound for the summation of $\phi_g$s for the given step sizes of $\forall g \in [G]: \mu_g \leq \left(1 + c_{0g}\frac{\omega_{0g} + \tau}{\sqrt{n_g}} \right)^{-1}/2\sqrt{n n_g}$:
	\begin{align}	 
	\nr 
	\sum_{g=1}^{G} \sqrt{\frac{n_g}{n}} \phi_g\left(\mu_g\right) 
	&\leq \sum_{g=1}^{G} \sqrt{\frac{n_g}{n}} \mu_g n_g \left(1 + c_{0g}\frac{\omega_{0g} + \tau}{\sqrt{n_g}} \right)
	\\ \nr 
	&\leq \frac{1}{2} \sum_{g=1}^{G} \frac{n_g}{n} = \frac{1}{2}
	\end{align}
	Therefore, Condition 1 reduces to $\rho_0\left(\mu_0\right) \leq 1/2$, which is satisfied with high probability if we have enough \textit{total} number of samples as shown below. Replacing high probability upper bound of $\rho_0$ from Lemma \ref{lemm:hpub}, we have the condition as:	
	\begin{align}	 
	\nr 
	\rho_0\left(\mu_0\right)
	&\leq  \frac{1}{2} \left[1 - \mu_0 n \left(1 - \sqrt{2} c_0\frac{2 \omega_0 + \tau}{ \sqrt{n}} \right)\right] \leq \frac{1}{2}
	\end{align}
	For $n > 8 c_0^2 (\omega(\cA_0) + \tau/2)^2$ the term in parenthesis is positive and the inequality is satisfied for all $\mu_0 \geq 0$. 
		
	For Condition 2, we plug in upper bounds of Lemma \ref{lemm:hpub} and get the following high probability upper bound: 
	\begin{align}	
	\nr 
	&\rho_g\left(\mu_g\right) +  \sqrt{\frac{n}{n_g}} \frac{\mu_0}{\mu_g}\phi_g\left( \mu_g \right)
	\\ \nr
	&\hspace*{75pt}\leq	 \frac{1}{2} \left[1 - \mu_g n_g \left(1 - \sqrt{2} c_g\frac{2 \omega_g + \tau}{ \sqrt{n_g}} \right)\right]  
	\\ \nr 
	&\hspace*{75pt}+  \sqrt{nn_g}\mu_0 \left(1 + c_{0g}\frac{\omega_{0g} + \tau}{\sqrt{n_g}} \right)
	\leq 1
	\end{align} 
	Thus, Condition 2 becomes: 		
	\begin{align}	
	\nr 
	\sqrt{2} c_g\frac{2 \omega_g + \tau}{ \sqrt{n_g}}
	\leq 1 + \frac{1}{\mu_g n_g} -  2\sqrt{\frac{n}{n_g}}\frac{\mu_0}{\mu_g} \left(1 + c_{0g}\frac{\omega_{0g} + \tau}{\sqrt{n_g}} \right)
	\end{align} 
	Note that using this condition, we want to determine the sample complexity for each group $g$, i.e., lower bounding $n_g$, for the given step sizes. Therefore, we can replace the lower bound for $1/\mu_g > 2\sqrt{n n_g} \left(1 + c_{0g}\frac{\omega_{0g} + \tau}{\sqrt{n_g}} \right)$ in the second term of the RHS and get the modified condition as:
	\begin{align}	
	\nr 
	\sqrt{2} c_g\frac{2 \omega_g + \tau}{ \sqrt{n_g}}
	\leq 1 + 2\sqrt{\frac{n}{n_g}} \left(1 + c_{0g}\frac{\omega_{0g} + \tau}{\sqrt{n_g}} \right) \left[1 - \frac{\mu_0}{\mu_g}\right]
	\end{align} 
	Because of the definition of $\mu_0$, we have $\forall g: \mu_0 \leq \mu_g$ and therefore the second term is positive for all $g$s which reduces Condition 2 to $\sqrt{2} c_g\frac{2 \omega_g + \tau}{ \sqrt{n_g}} \leq 1$ that lead to the sample complexity of $n_g > 8c^2_g (\omega(\cA_g) + \tau/2)^2$, which completes the proof.
	\hfill {\qedsymbol}
\begin{figure}[t!]
	\centering
	\begin{subfigure}[b]{0.2175\textwidth}
		\includegraphics[width=\textwidth]{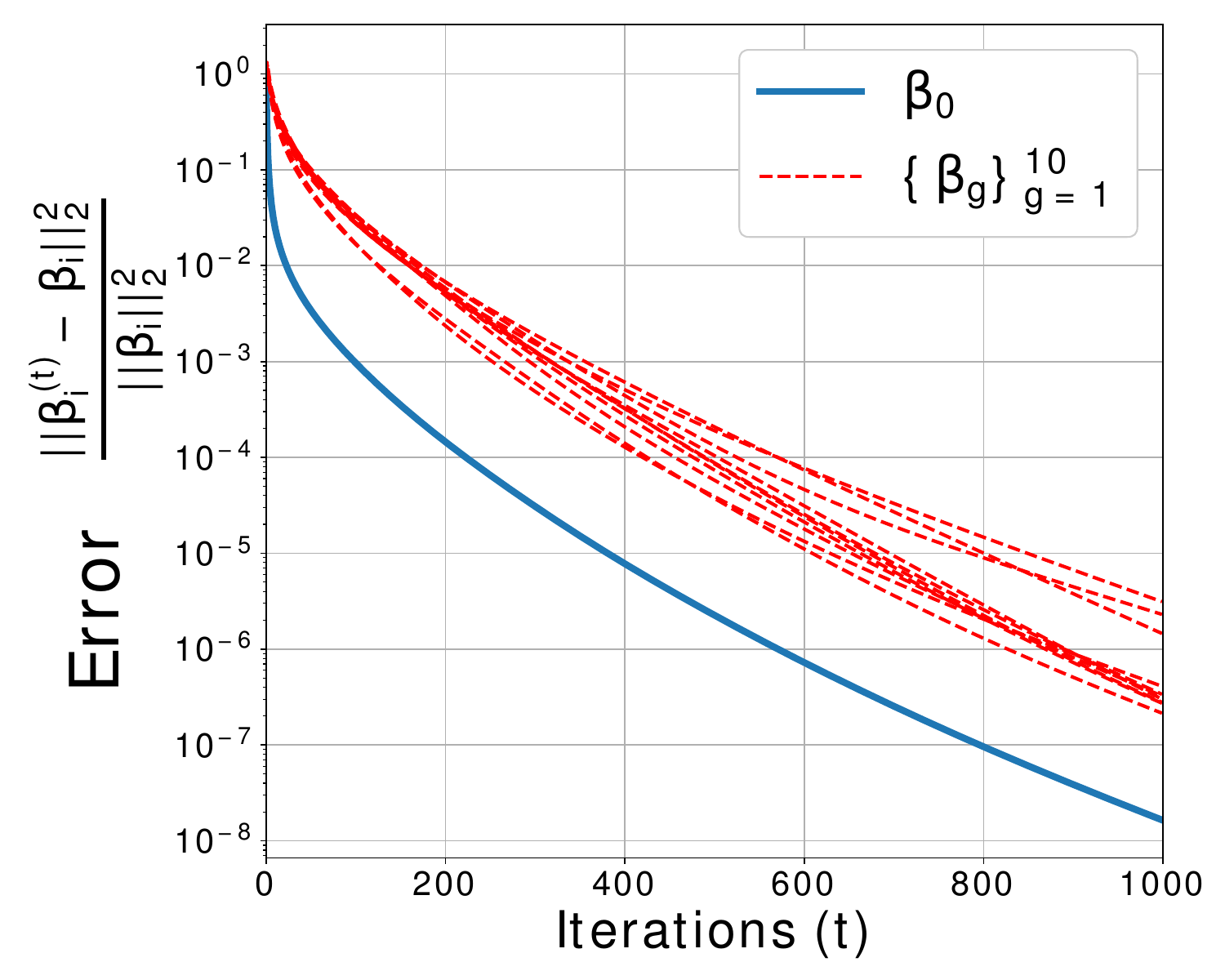}				
		\caption{$n_g = 60$, $\w = 0$}\label{fig syn1a}
	\end{subfigure} ~
	\begin{subfigure}[b]{0.2175\textwidth}
		\includegraphics[width=\textwidth]{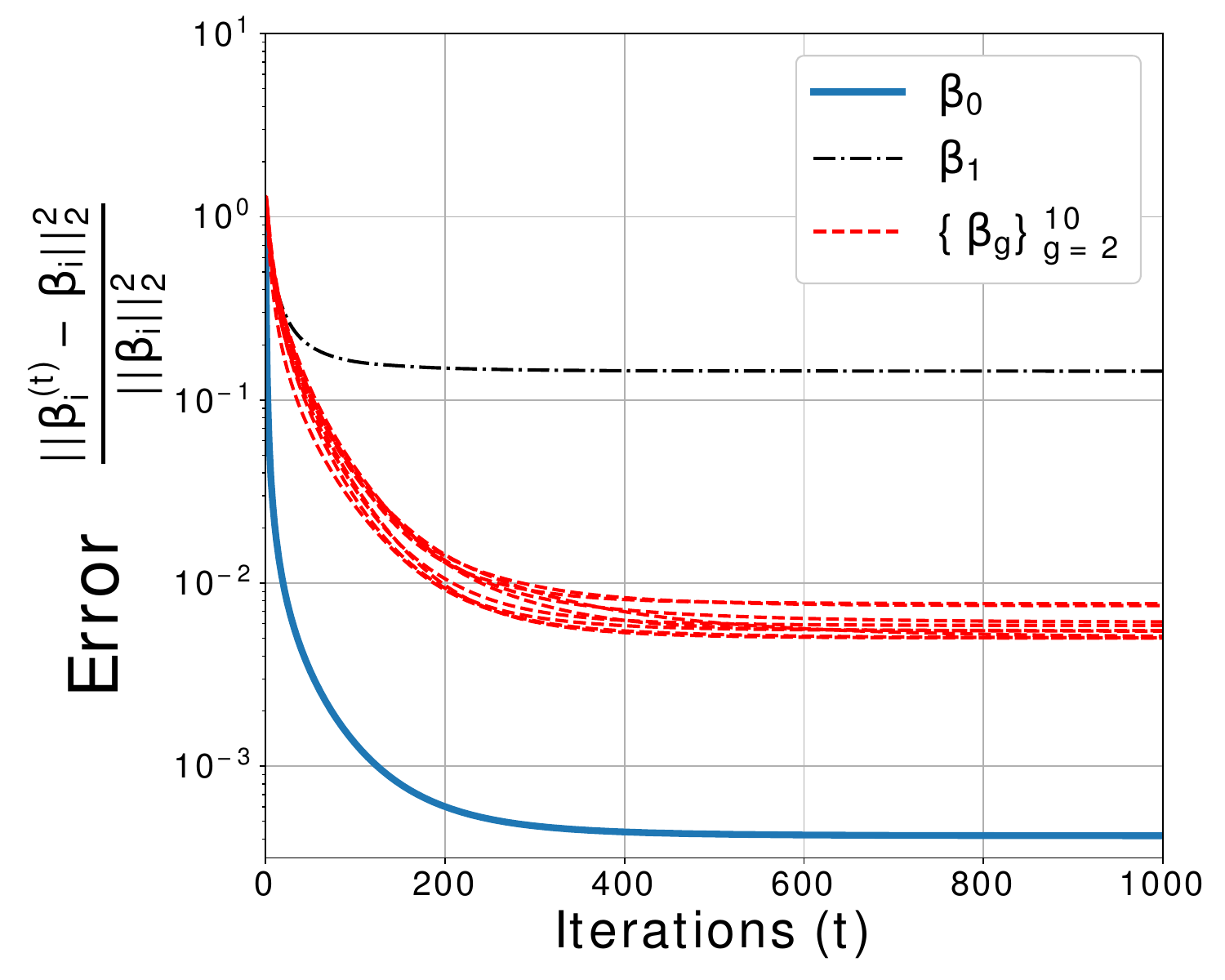}				
		\caption{$n_g = 60$, $\w_1 \neq 0$}\label{fig syn1b}
	\end{subfigure}
	\begin{subfigure}[b]{0.2175\textwidth}
		\includegraphics[width=\textwidth]{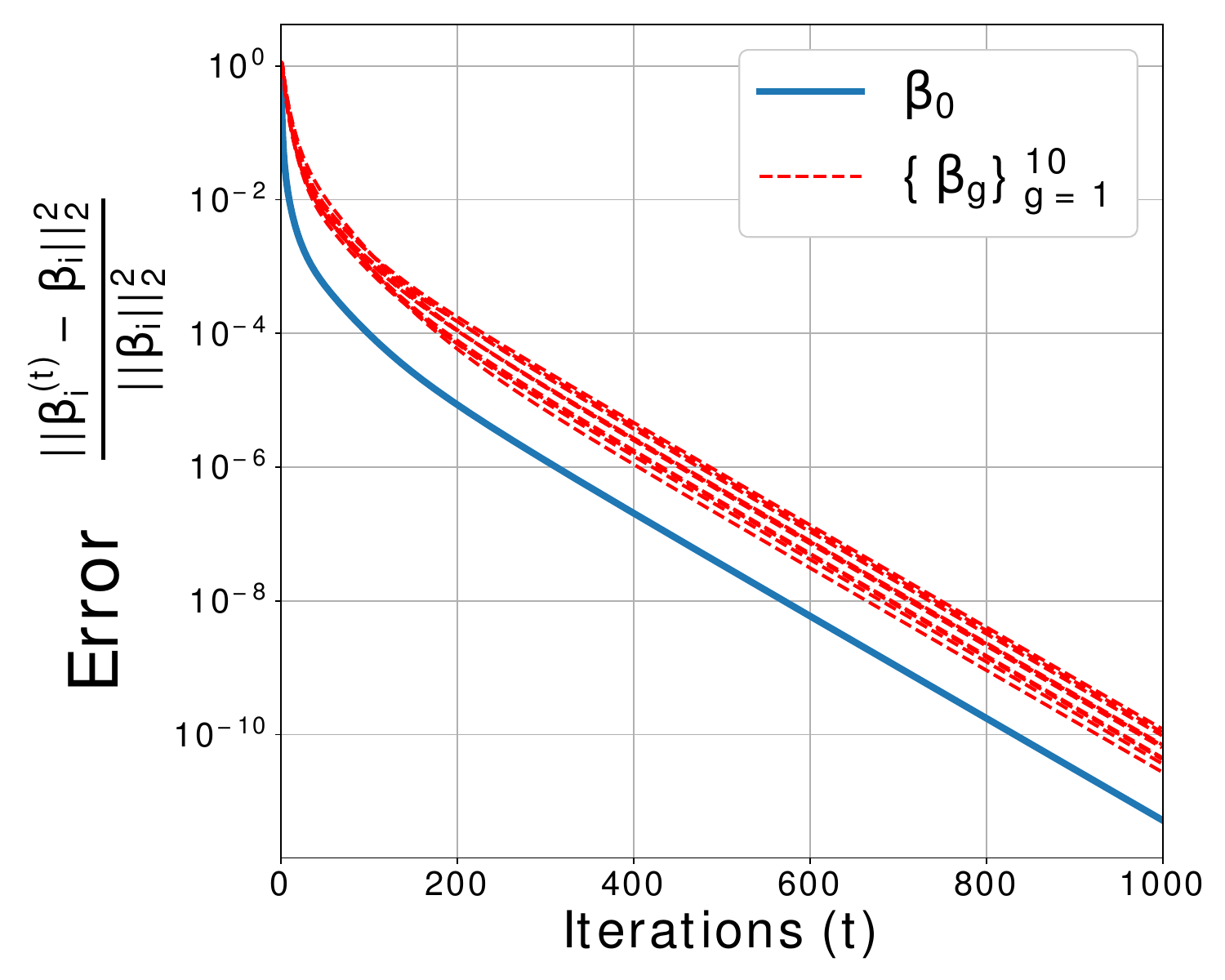}
		\caption{$n_g = 150$, $\w = 0$} \label{fig syn2a}
	\end{subfigure} ~
	\begin{subfigure}[b]{0.2175\textwidth}
		\includegraphics[width=\textwidth]{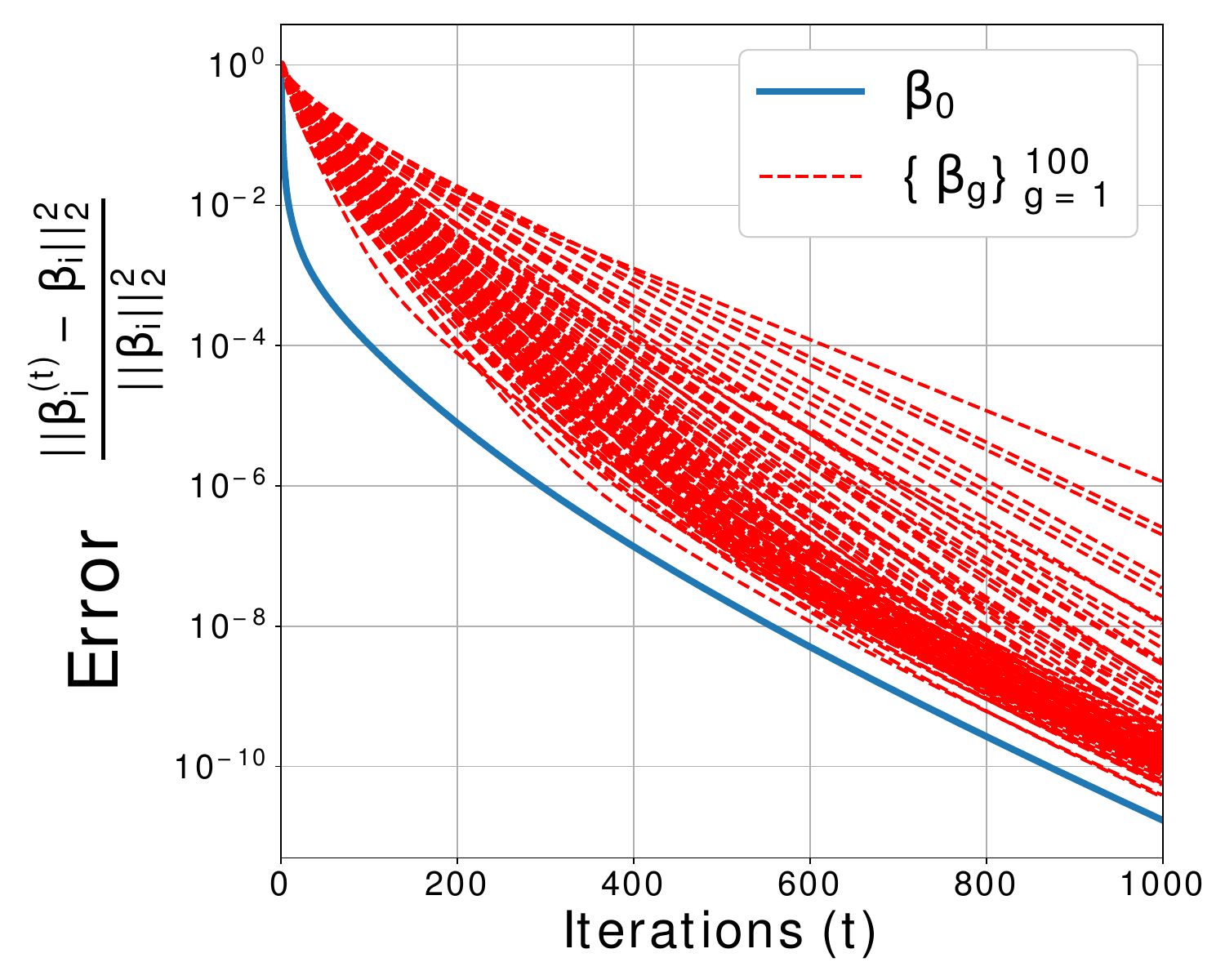}
		\caption{$n_g = 150$, $\w = 0$}\label{fig syn2b}
	\end{subfigure}
	\caption{\small In (a), (b), and (c)  experiments $p = 100$, $G = 10$, $\forall g \in [G]: s_g = 10$, and $s_0 = p$. For (d) $p = 1000$, $G = 100$, $\forall g \in [G]: s_g = 10$, and $s_0 = 100$. (a) Noiseless fast convergence. (b) Noise on the first group does not impact other groups as much. (c) Increasing sample size improves rate of convergence. (d) \dc\ convergences fast even with a large number of groups $G=100$.}
	\label{fig syn12}
\end{figure}
\section{Experiments on Synthetic Data}
\label{sec:expds}
We considered sparsity based simulations with varying $G$ and sparsity levels. In our first set of simulations, we set $p=100$, $G=10$ and sparsity of the individual parameters to be $s=10$. We generated a dense $\bbeta_0$ with $\|\bbeta_0\|=p$ and did not impose any constraint. Iterates $\{\bbeta^{(t)}_g\}_{g=1}^G$ are obtained by projection onto the $\ell_1$ ball $\|\bbeta_g\|_1$. Nonzero entries of $\bbeta_g$ are generated with ${\cal{N}}(0,1)$ and nonzero supports are picked uniformly at random. Inspired from our theoretical step size choices, in all experiments, we used simplified learning rates of $\frac{1}{n}$ for $\bbeta_0$ and $\frac{1}{\sqrt{nn_g}}$ for $\bbeta_g$, $g \in [G]$. Observe that, cones of the individual parameters intersect with that of $\bbeta_0$ hence this setup actually violates \ds\ (which requires an arbitrarily small constant fraction of groups to be non-intersecting). Our intuition is that the individual parameters are mostly incoherent with each other and the existence of a nonzero perturbation over $\bbeta_g$'s that keeps all measurements intact is unlikely. Remarkably, experimental results still show successful learning of all parameters from small amount of samples. We picked $n_g=60$ for each group. Hence, in total, we have $11p=1100$ unknowns, $200=G\times 10+100$ degrees of freedom and $G\times 60=600$ samples. In all figures, we study the normalized squared error $\frac{\|\bbeta^{(t)}_g-\bbeta_g\|_2^2}{\|\bbeta_g\|_2^2}$ and average $10$ independent realization for each curve. Fig. \ref{fig syn1a} shows the estimation performance as a function of iteration number $t$. While each group might behave slightly different, we do observe that all parameters are linear converging to ground truth.
	
In Fig. \ref{fig syn1b}, we test the noise robustness of our algorithm. We add a ${\cal{N}}(0,1)$ noise to the $n_1=60$ measurements of the first group \emph{only}. The other groups are left untouched. While all parameters suffer nonzero estimation error, we observe that, the global parameter $\bbeta_0$ and noise-free groups $\{\bbeta_g\}_{g=2}^G$ have substantially less estimation error. This implies that noise in one group mostly affects itself rather than the global estimation. In Fig. \ref{fig syn2a}, we increased the sample size to $n_g=150$ per group. We observe that, in comparison to Fig. \ref{fig syn1a}, rate of convergence receives a boost from the additional samples as predicted by our theory.

Finally, Fig. \ref{fig syn2b} considers a very high-dimensional problem where $p=1000$, $G=100$, individual parameters are $10$ sparse, $\bbeta_0$ is $100$ sparse and $n_g=150$. The total degrees of freedom is $1100$, number of unknowns are $101000$ and total number of datapoints are $150\times 100=15000$. While individual parameters have substantial variation in terms of convergence rate, at the end of $1000$ iteration, all parameters have relative reconstruction error below $10^{-6}$.

\section{Conclusion}
We presented an estimator for the joint estimation of common and individual parameters of the data sharing model. 
We showed that the sample complexity for estimation of the shared parameter depends on the total number of sample $n$.
In addition, the shared parameter error rate decays as $1/\sqrt{n}$. 
These results indicate that our estimator benefits from the pooled data in estimating the common parameters. 
Both sample complexity and upper bound of error depend on the \emph{maximum} Gaussian width among the spherical caps induced by the error cones of all parameters.
We provided a projected gradient descent algorithm for estimation of the parameters and analyzed its convergence rate and showed geometric convergence for a carefully selected step size. Finally, we complemented the theoretical results presented with simulation results.

%
%
%
\section*{Acknowledgment}
The research was supported in part by NSF grants OAC-1934634, IIS-1908104, IIS-1563950, IIS-1447566, IIS-1447574, IIS-1422557, and CCF-1451986. Samet Oymak is partially supported by the NSF award CNS-1932254. Amir Asiaee and Kevin Coombes would like to thank the Mathematical Biosciences Institute (MBI) at Ohio State University, for partially supporting this research through NSF grants DMS 1440386 and DMS 1757423.


\ifCLASSOPTIONcaptionsoff
  \newpage
\fi

\bibliographystyle{IEEEtran}
\bibliography{IEEEabrv,nips18}


%
%
%
%
%
%

\end{document}